\documentclass{article}

\PassOptionsToPackage{square, sort, comma, numbers, compress}{natbib}

\usepackage[preprint]{neurips_2022}

\usepackage[utf8]{inputenc} 
\usepackage[T1]{fontenc}    
\usepackage{amsfonts}
\usepackage{hyperref}       
\usepackage{url}            
\usepackage{booktabs}       
\usepackage{amsfonts}       
\usepackage{nicefrac}       
\usepackage{microtype}      
\usepackage{xcolor}         
\usepackage[pdftex]{graphicx}
\usepackage{bbm}
\usepackage{wrapfig}
\usepackage{subcaption}
\usepackage{float} 

\usepackage{amsthm}

\newtheorem{assumption}{Assumption}

\newenvironment{lemma}[1]{\par\noindent\textbf{Lemma}:\space#1}{}

\usepackage{amsmath,amsfonts,bm}

\def\eqref#1{equation~\ref{#1}}

\def\1{\bm{1}}

\DeclareMathAlphabet{\mathsfit}{\encodingdefault}{\sfdefault}{m}{sl}
\SetMathAlphabet{\mathsfit}{bold}{\encodingdefault}{\sfdefault}{bx}{n}

\title{Sample-Efficient Optimisation with Probabilistic Transformer Surrogates}

\author{%
  Alexandre Maraval\thanks{Authors contributed equally} \\
  Huawei R\&D London\\
  \texttt{alexandre.maraval@huawei.com} \\
  \And
  Matthieu Zimmer\thanks{Authors contributed equally} \\
  Huawei R\&D London\\
  \texttt{matthieu.zimmer@huawei.com} \\
  \And
  Antoine Grosnit\thanks{Authors contributed equally} \\
  Huawei R\&D London\\
  \texttt{antoine.grosnit@huawei.com} \\
  \And
  Rasul Tutunov \\
  Huawei R\&D London\\
  \texttt{rasul.tutunov@huawei.com} \\
  \And
  Jun Wang \\
  University College London\\
  \texttt{jun.wang@cs.ucl.ac.uk} \\
  \And
  Haitham Bou Ammar\\
  Huawei R\&D London\\
  University College London\\
  \texttt{haitham.ammar@huawei.com} \\
}

\begin{document}

\maketitle

\begin{abstract}
  
  Faced with problems of increasing complexity, recent research in Bayesian Optimisation (BO) has focused on adapting deep probabilistic models as flexible alternatives to Gaussian Processes (GPs). 
  In a similar vein, this paper investigates the feasibility of employing state-of-the-art probabilistic transformers in BO. 
  Upon further investigation, we observe two drawbacks stemming from their training procedure and loss definition, hindering their direct deployment as proxies in black-box optimisation. First, we notice that these models are trained on uniformly distributed inputs, which impairs predictive accuracy on non-uniform data - a setting arising from any typical BO loop due to exploration-exploitation trade-offs. Second, we realise that training losses (e.g., cross-entropy) only \emph{asymptotically} guarantee accurate posterior approximations, i.e., after arriving at the global optimum, which generally cannot be ensured. At the stationary points of the loss function, however, we observe a degradation in predictive performance especially in exploratory regions of the input space. To tackle these shortcomings we introduce two components: 1) a \textit{BO-tailored training prior} supporting non-uniformly distributed points, and 2) a novel approximate posterior regulariser trading-off accuracy and input sensitivity to filter favourable stationary points for improved predictive performance. In a large panel of experiments, we demonstrate, for the first time, that \text{one} transformer pre-trained on data sampled from \emph{random GP priors} produces competitive results on 16 benchmark black-boxes compared to GP-based BO. Since our model is only pre-trained once and used in all tasks \textit{without} any retraining and/or fine-tuning, we report an order of magnitude time-reduction, while matching and sometimes outperforming GPs. 

\end{abstract}



  

\section{Introduction}\label{sec:intro}
Black-box optimisation focuses on objective functions not available in closed form and has proven to be of paramount importance in many engineering fields including, but not limited to, automatic machine learning \citep{2012_Snoek, 2020_Grosnit,HeboPaper}, drug discovery \citep{AntBO2022, Mole_paper, DrugPaper_two}, chemical engineering \citep{2020_Griffiths, LSBO_Antoine_2021, NEURIPS2020_81e3225c}, and  experimental design \citep{NEURIPS2019_d55cbf21, 8957442,ExpDesignpaper_two}. Bayesian optimisation (BO) methods provide a principled solution for black-box optimisation. Despite a vast variety of suggested BO algorithms \citep{ReviewBO,AntBO2022,HeboPaper}, these methods tend to share two main components: 1) a probabilistic surrogate model of the objective, and 2) a strategy for querying novel promising probes to evaluate.   

Gaussian Processes (GPs) are a popular choice of probabilistic surrogates in BO due to their sample-efficiency and ability to quantify uncertainty \citep{2006_Rasmussen}. BO methods equipped with GP surrogates effectively leverage posterior uncertainties to achieve a neat balance between exploration and exploitation, resulting in  sample-efficient performance.
Although traditional GPs deal with homoscedastic scenarios and scale poorly (as cubic polynomial) with the number of inputs, various extensions have been suggested to tackle those issues; see \citep{HeboPaper,NIPS2005_4491777b,HeteroGP} and Section \ref{sec:background_and_related_work} for details.         

Concurrently with the successes of GPs, recent research has focused on uncertainty-aware neural networks with the promise of more flexible and scalable probabilistic models \citep{2015_Snoek,1997_Bishop,1992_Neal,2015_Blundell,2020_Jospin}. These efforts include but are not limited to new architectures for neural networks \cite{2015_Snoek}, parametrised priors over the network's weights \citep{1997_Bishop,1992_Neal,2015_Blundell,2020_Jospin}, dropout regularisation techniques \citep{2012_Hinton, 2014_Srivastava}, deep ensembles \citep{2017_Lakshminarayanan, 2019_Fort}, neural processes \citep{2018_Garnelo, 2018_Garnelo_a, 2021_Chakrabarty} and probabilistic transformers (PTs)  \citep{2017_Vaswani, 2021_Muller}. Apart from achieving a superior modelling performance in a variety of supervised machine learning tasks \citep{2020_Dosovitskiy,2020_Touvron,2019_Ye,2019_Sun}, 
\citet{2021_Muller} show that they can successfully learn to approximate the posterior of Bayesian models by training on artificial data, sampled from those very models.



Given the impressive results in \citep{2021_Muller}, we wonder if such a model can easily be ported to BO to replace GP surrogates while ensuring sample-efficiency. During our early attempts, we realised two drawbacks prohibiting a direct application to BO. In this paper, we describe these two issues, propose resolution to each, and demonstrate, for the first time to the best of our knowledge, sample-efficient optimisation with PT surrogates. 

\paragraph{Contribution I: Drawbacks of PTs when applied to BO.} As a first contribution, we realise that the method in \citep{2021_Muller} is challenged by two properties commonly observed during BO. Being trained on uniformly spaced data, PTs struggle to make predictions when non-uniformly distributed observations are available; a common situation faced in BO. We also notice that those transformers require many observations in novel regions before adapting their posterior prediction. Such behaviour is detrimental in BO as it can easily increase sample complexity, requiring lots of black-box function evaluations. 

\paragraph{Contribution II: Non-uniform priors and novel regularisation.} Having identified these problems, we then contribute by proposing: 1) a simple yet effective non-uniform sampling procedure, and 2) a novel regulariser that trades-off accuracy and sensitivity during pre-training (Section \ref{sec:pt4bo}). Our non-uniform sampler emphasises important cases for prediction by weighting sampling probabilities in accordance to target values. The regulariser, furthermore, induces a stationarity structure to enable better generalisation and faster adaptation to novel input locations.  

\paragraph{Contribution III: One randomly pre-trained PT can do sample-efficient optimisation.} As a third contribution of our work, we demonstrate that a transformer pre-trained with our components on data collected from \emph{random hyper-priors} can optimise a variety of benchmark black-box functions, \emph{without the need for any fine-tuning or re-training}. Results on 38 standard BO benchmark tasks indicate that our method is indeed sample efficient in that it recovers or outperforms GPs, while being an order of magnitude faster in wall-clock time. Additionally, we demonstrate that our contributions to the pre-training of PTs lead to significant improvements in BO performance (e.g., going from 36\% improvement over random search to 72.1\% on 10 dimensional tasks).




\section{Background \& related work}\label{sec:background_and_related_work}
\subsection{Bayesian optimisation}\label{subsec:background:bo}
In BO \citep{1964_Kushner, 1978_Mockus, 1975_Zhilinskas, 2010_Brochu, 2020_Grosnit, 2012_Snoek}, we optimise a function $f(\cdot)$ solving: 
\begin{equation}
    \label{eq:bo}
    \bm{x}^{\star} = \arg\min_{\bm{x} \in \mathcal{X}} f(\bm{x}),
\end{equation}
where $f :\mathcal{X} \rightarrow\mathbb{R}$ is a black-box function over a domain $\mathcal{X} \subset \mathbb{R}^d$. 
BO is an especially attractive methodology when data efficiency is important, for example when optimising functions that are expensive to evaluate, functions over which we have no prior knowledge, and/or functions with no available gradients. 
There are two core components in BO; a probabilistic surrogate model and an acquisition function that trades-off exploration and exploitation to determine the next evaluation query $\bm{x}^{\text{new}} \in \mathcal{X}$. Among many possible surrogates, GPs \citep{2006_Rasmussen, 2009_Osborne} are a popular alternative as they are data-efficient and provide calibrated prediction uncertainties. Moreover, many acquisition functions with varying degrees of success have been proposed in the literature \citep{1978_Mockus, 1998_Jones, 2010_Srinivas, 1964_Kushner, 2016_Gonzlez}. In this paper, we utilise expected improvement (EI) acquisitions \citep{1978_Mockus, 1998_Jones}, but note that our framework is compatible with any other from the literature \citep{2010_Brochu, 2009_Osborne}. 

Although successful in many instances \citep{2012_Snoek, 2020_Griffiths, 2021_Turner, AntBO2022}, the increasing complexity of the application domains targeted with BO has motivated researchers to propose deep-learning-based probabilistic surrogates and novel BO forms which we briefly review next. 

\subsection{Related work: deep-surrogates \& non-standard BO forms}
We first consider the
use of deep neural networks (DNN) in GP design to tackle specific BO modelling problems (heteroscedasticity, structured search spaces, etc.). Then, we give an overview of the different methods that have been explored for deep networks to perform Bayesian inference. Finally, we discuss the use of those models as probabilistic surrogates in BO. 

\paragraph{Non-standard GPs in BO.}
While standard GPs are widely used in practice for their calibrated uncertainty and data-efficiency, there are applications calling for more flexible or scalable models. Standard GP surrogate assumes that the target function $f$ has stationary variations,  
therefore an important line of research consists in adapting GPs to functions with non-stationary variations, or functions observed with a non-homogeneous noise (e.g., heteroscedastic noise). 
To tackle the variability of the noise scale, an extra GP can be used, forming a heteroscedastic model \citep{2011_azaro, 2012_Wang, 2017_Calandra}. Other approaches to cope with general model misspecifications consist in learning a transformation of the GP inputs and/or outputs along with the usual GP hyperparameters. These transformations can be as simple as box-cox transformation \citep{1964_Box, HeboPaper}, or obtained through a cascade of GPs leading to a deep-GPs \citep{2013_Damianou}. Neural networks have also been considered to enhance modelling flexibility, for instance by using normalising flows to transform the GP likelihood \citep{2021_maronas}. 
Moreover, when optimising a black-box function defined over structured search space, GP can still be a valid model either by using a specific kernel adapted to the structured inputs \citep{moss2020boss}, or by getting a continuous latent representation of the inputs, e.g. with a variational-autoencoder \citep{2014_Kingma}, and by running BO in the latent space \citep{2016_Gomez, 2020_Griffiths, 2020_Tripp, 2020_Grosnit}. 
On the scalability side, escaping from the cubic complexity of standard GP inference with respect to the number of observations has been an active line of research, leading to the development of sparse GPs  \citep{2005_Snelson, 2011_lazaro, 2013_Hensman, 2015_Wilson}.
Instead of performing exact Bayesian inference, sparse GPs rely on pseudo-input-output points, known as inducing points, and approximate the model posterior through variational inference.  

\paragraph{Neural Bayesian models.}
Efforts have been made to benefit from the high flexibility and scalability of DNNs without suffering from their difficulty to provide reliable uncertainty estimations. 
In \citep{2015_Snoek}, the penultimate layer of a fully-connected DNN trained to make deterministic predictions is used to provide the basis functions of a Bayesian linear regression (BLR). The Bayesian modelling exploits features extracted from the frozen penultimate layer for uncertainty estimation. Of course, the quality of the estimation is limited by the fact that most of the DNN is fully deterministic. One way to extend Bayesian inference to the entire DNN consists in imposing a parametrised prior distribution over the weights and activations of the network arriving at Bayesian neural networks (BNNs) \citep{1997_Bishop,1992_Neal,2015_Blundell,2020_Jospin}.
By using a stochastic model and performing Bayesian inference through MCMC sampling or with a variational inference approximation, BNNs can provide uncertainty estimations and prevent overfitting. 
The dropout regularisation \citep{2012_Hinton, 2014_Srivastava}, commonly applied during DNNs training, is also used to obtain distributions for the output predictions at test time, in a way that can be framed as Bayesian inference \citep{2016_Gal}. 
Similarly, deep ensembles (DEs) \citep{2017_Lakshminarayanan, 2019_Fort}, which make predictions using a set of DNNs trained from different initial parameters, are also performing some form of Bayesian inference \citep{2020_Wilson, 2021_Angelo}. 
Even though DEs often outperform BNNs and dropout regularisation \citep{2020_Caldeira, 2020_Scalia, 2020_Gustafsson}, recent works \citep{2021_Rahaman, 2020_Ashukha} show that, especially in low-data regime, DEs still suffer from under-calibration.
Neural Processes (NPs) \citep{2018_Garnelo, 2018_Garnelo_a} are another popular DNN-based model adopting the Bayesian framework. Combining the uncertainty estimations from GPs and the scalability of neural networks, they implicitly learn a kernel function from the observed data, allowing to predict output distributions at test points, given context points. \citet{2018_kim} observes that NPs tend to underfit at context points, which they remedy by introducing an attention mechanism, which improves the fit as each prediction is made by focusing on more relevant context points. 
Attention mechanism is also exploited by the PT model designed by \citet{2021_Muller}, which shows improved Bayesian inference ability compared to Attentive NPs. We give more details on PT in Section~\ref{subsec:background:pt} as it is the model we build upon.

\paragraph{Neural Bayesian models for BO.}
Most of the methods using DNNs to perform Bayesian inference have been chosen as surrogate in the context of BO. \citet{2019_Ma} adopt Graph-BNN as a surrogate model to run BO on a neural architecture search problem involving a structured search space. GP scalability issues have motivated the use of DEs \citep{2021_Lim, 2021_Gruver}, of BLR-augmented neural networks  \citep{2012_Snoek}, or of BNNs \citep{2016_Springenberg}. \citet{2016_Springenberg} directly trains a BNN using a robust stochastic gradient Hamiltonian Monte Carlo (BOHAMIANN) that allows for better posterior sampling and acquisition function estimation.
NPs have also been considered in the context of BO \citep{2021_Shangguan, 2020_Luo, 2021_Chakrabarty}.
\citet{2021_Shangguan} take advantage of the scalability and modelling performance of standard NPs, performing competitively compared to GP-based BO as well as BOHAMIANN. 
Finally, conditional and attentive NPs \citep{2020_Luo,2021_Chakrabarty} have respectively shown promising results on multi-task BO and batch-BO.
To the best of our knowledge, we are the first to attempt adapting PT to the BO framework.


\paragraph{Conclusions from previous literature and scope of our paper.}
In this paper, we aim at tackling standard black-box optimisation problems with a Bayesian neural network model, and leave the more exotic BO forms for later work. The choice of PT as our base DNN comes from its potential ability to perform Bayesian inference without needing to be fine-tuned on black-box observations. However, this general ability was demonstrated in an ideal case and therefore may not be straightforwardly extendable to the BO scenario, as we show in Section \ref{sec:drawbacks}.

\subsection{Probabilistic transformers}\label{subsec:background:pt}
Transformer models \citep{2017_Vaswani} have achieved a state-of-the-art modelling performance in a variety of supervised machine learning tasks including but not limited to, computer vision \citep{2020_Dosovitskiy,2020_Touvron,2019_Ye,2019_Sun}, machine translation \citep{2020_Radford,2020_Yang,2018_Radford,2019_Liu}. 
Inspired by those successes, recent research focused on  creating transformers capable of modelling uncertainty. Here, we differentiate a handful of approaches \citep{2021_Xue, 2020_Hendrycks, 2018_Tran} designed for specific applications (e.g., natural language processing and understanding). Those methods, however, require large amounts of (labeled) training data and, as such, can incur high sample complexity when used in conjunction with BO - the focus of our paper.

In their seminal work in \citep{2021_Muller}, the authors focus on learning to approximate Bayesian posteriors while requiring \emph{no task-specific data}, i.e., data from the black-box under optimisation in our case. While the scope of applications of this model reaches far-beyond regression, the authors demonstrate that those transformers efficiently approximate GP posteriors in the context of \emph{standard} regression tasks. Hence, we believe that such a model  constitutes a promising direction for sample efficient optimisation. 

Attempting to directly port the work in \citep{2021_Muller} to BO, we realised two drawbacks stemming from the exploratory nature of the optimisation procedure. Before detailing those shortcomings in Section~\ref{sec:drawbacks}, we first present the PT needed for the remainder of the paper. 

\paragraph{Transformers can do Bayesian inference.} 
\citet{2021_Muller} introduce a transformer architecture $q_{\bm{\phi}}(\cdot)$ trained to approximate posterior distributions.
Given a set of $n_{\text{obs}}$ observations $D_{\text{obs}}=\{\bm{x}^{\text{obs}}_{i}, y^{\text{obs}}_i\}_{i=1}^{n_{\text{obs}}}$ and $n_{\text{pred}}$ locations at which we desire to make predictions $ \bm{x}^{\text{pred}}_{1:n_{\text{pred}}}$, the transformer outputs a distribution that approximates the true posterior through a forward pass. To yield competitive results with deep-networks, \citet{2021_Muller} adapt output-domain discretisation \citep{bellemare2017distributional} yielding Riemann (or bar-plot pdf) distributions. Precisely, given a set of $n_{\text{buckets}}$ ``buckets'', $\bm{b} = \{b_l\}_{l=1}^{n_{\text{buckets}}}$, discretising the output domain, the transformer learns to classify  $y^{\text{pred}}_j$ belonging to a bucket $b_l$, i.e., $\mathbb{P}_{\bm{\phi}}(y^{\text{pred}}_j \in b_l)$ for all $l\in[1:n_{\text{buckets}}]$.  


To train the transformer, artificially labeled datasets $D=\{(\bm{x},\bm{y})\}$ are repeatedly sampled from a \textit{dataset prior} $D \sim p(\mathcal{D})$, e.g., given by GPs or BNNs \citep{2021_Muller}. In this paper, we consider $p(\mathcal{D})$ to be given by a hyper-GP with priors over the kernel hyperparameters. Specifically, we utilise  GPs with zero means and Matern-5/2 kernels.
We then instantiate a GP $f$ with random  kernel parameters $\bm{\theta}$ sampled from a \emph{hyper-prior} $\bm{\theta} \sim p(\bm{\Theta})$. We then form the dataset $D$ by sampling a collection of $N$ inputs in $\mathcal{X} \subseteq \mathbb{R}^d$ and their associated observations $\bm{y}$ from the prior of this specific GP instance such that: $\bm{y}\sim f(\cdot|\bm{x},\bm{\theta})\in\mathbb{R}^N$.
The distribution from which parameters $\bm{\theta}$ are sampled and the kernel of the GPs fully specify the \textit{dataset prior} $p(\mathcal{D})$.

We then randomly split the data in $\mathcal{D}$ into ${D}_{\text{obs}}$ and points to be predicted $\langle \bm{x}_{1:n_{\text{pred}}}^{\text{pred}}, y_{1:n_{\text{pred}}}^{\text{pred}} \rangle$ with $y_{1:n_{\text{pred}}}^{\text{pred}}$ being the true targets. Following \citep{2021_Muller}, we learn the parameters of the transformer $\bm{\phi}$ by minimising the cross-entropy (CE) loss, which is defined for \emph{one} prediction data point as $\langle \bm{x}_j^{\text{pred}}, y_{j}^{\text{pred}}\rangle$: 
\begin{equation}
\label{eq:ce_loss}
    \mathcal{J}(\bm{\phi}) = \mathbb{E}_{(\langle \bm{x}_j^{\text{pred}}, y_{j}^{\text{pred}}\rangle, {D}_{\text{obs}})\sim p(\mathcal{D})}[-\log q_{\bm{\phi}}(\cdot|\bm{x}_j^{\text{pred}}, D_{\text{obs}})].
\end{equation}


Once we solve the problem in Equation \ref{eq:ce_loss}, we can easily infer predictive distributions on observations of interest (e.g., the true black-box in our case) by executing a forward pass in the network.  
Next, we describe two shortcomings of the method in \citep{2021_Muller} prohibiting its direct application to BO. We then elaborate potential resolutions in Section \ref{sec:pt4bo}. 



\section{Drawbacks of probabilistic transformers for Bayesian optimisation}\label{sec:drawbacks}
One can reasonably wonder if the model described in Section \ref{subsec:background:pt} model could be applied as is in situations where a GP is usually the model of choice, e.g., in BO.
In this section, we make two observations that, we argue, are drawbacks of the approach of \citet{2021_Muller} when applied to sample-efficient optimisation.

\subsection{Uniformly distributed inputs are unrealistic in BO settings}\label{subsec:drawbacks:non_uniform}
We first notice that the PT is trained on datasets $D$ which have uniformly distributed input locations. Precisely, we create $D$ by uniformly sampling $\bm{x}$ in $\mathcal{X}$ and then sampling $\bm y$ from a GP prior $\bm{y}\sim f(\cdot|\bm{x},\bm{\theta})$.
This creates two problems considering the application to BO. 

First, the PT is trained on a space that is uniformly covered causing the predicted locations to be also uniformly distributed in $\mathcal{X}$.
Second, the conditioning locations (i.e. the observed points used to condition the predictions) are also uniformly distributed in $\mathcal{X}$, and thus uniformly close to the predicted ones.

However, these two situations hardly ever occur during BO.
In fact, as the search alternates between exploration and exploitation, the surrogate model is used to make predictions at locations in $\mathcal{X}$ that are either far away from the conditioning points (exploration step) or close to them (exploitation step).
\begin{figure}[htp]
\centering
\includegraphics[width=1.0\textwidth]{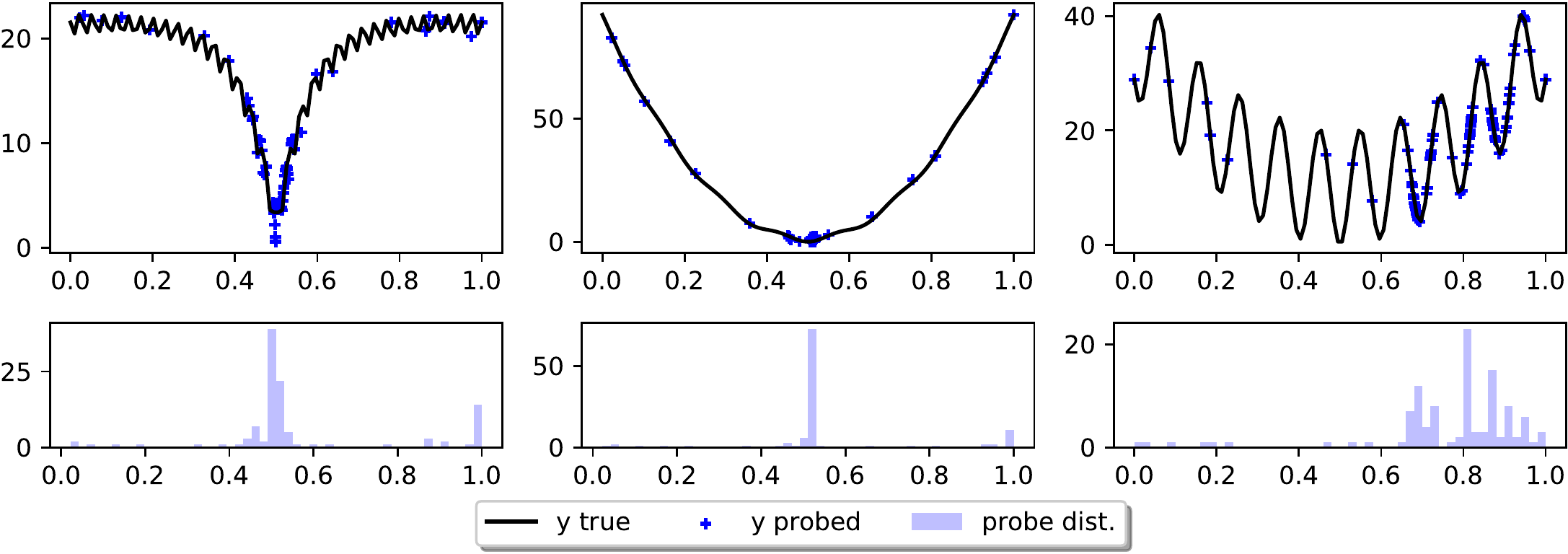}
\caption{GP predictions and BO run probed point distributions on Ackley, Griewank and Rastrigin in 1 dimension. Clearly, probed points during BO are far from being uniform in the search space.}
\label{fig:non_unif_probes}
\end{figure}
Providing credence to our argument, Figure \ref{fig:non_unif_probes} illustrates a real GP-based BO run on the Ackley, Griewank and Rastrigin functions in a 1 dimension search space.
The BO is run for 100 steps starting from 10 initial random points and the GP is retrained at each iteration.
The top part of each plot in Figure \ref{fig:non_unif_probes} shows the true function and the points acquired by the model during the optimisation.
The bottom part of each plot shows the distribution of input locations probed (i.e., evaluated in the black-box objective) during the BO search. 
It is clear from these plots that the distribution of probed points is non-uniform over the domain $\mathcal{X}$. 
Indeed, we observe that a large frequency of queries are made around the location where the model believes the optimiser lies in.
Respectively, points are probed with much lower frequency elsewhere.
This example supports our observation that PTs trained on uniformly sampled input data are not particularly geared towards non-uniform predictions needed during BO.

\subsection{PTs require many novel observations to adapt posteriors}\label{subsec:drawbacks:sensitivity}
For a given prior distribution $p(\mathcal{D})$, \citet{2021_Muller} showed that the true posterior predictive distribution $p(\cdot | \bm x, D)$ could be obtained under the assumption that the model is expressive enough to capture that true distribution.
However, in practice, it is often not the case and depends on many factors (e.g. how the histograms discretise the target space, how many layers are used, etc.).
Hence, the PT often converges to stationary points which can exhibit undesirable properties for BO.


A drawback we noticed that tends to increase the sample complexity in BO is that some of the trained PT models did not exhibit enough sensitivity to new observations as we elaborate next.
Consider the following example where we already observed points in $D$ and collect a new observation $(\bm x', y')$, the PT predicts the same posterior around $\bm x'$ \emph{with or without this new observation} even if $\bm{x}^{\prime}$ exhibited large distance to the current observations. 
We attribute such a behaviour to the fact that since during training, the model saw a large variety of functions, it needs a significant number of observations in a particular region before it is able to determine which information to use from its training datasets. 

\begin{figure}[htp]
\centering
\includegraphics[width=1.0\textwidth]{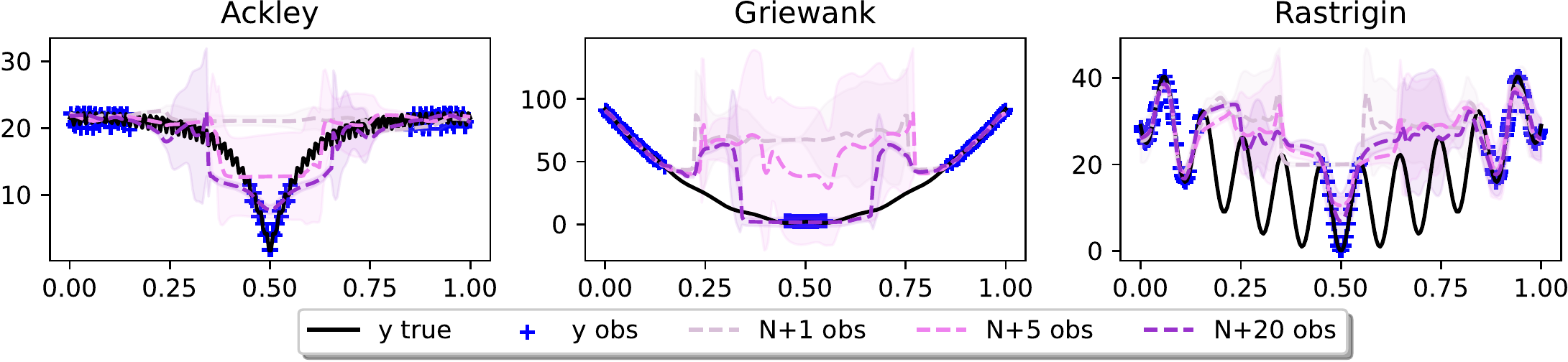}
\caption{PT predictive distribution with growing number of observations in unexplored regions. As we can see, the PT does change its posterior prediction but only upon observing many probes.}
\label{fig:pt_growing}
\end{figure}
It is easy to see why this behaviour is a drawback in a BO setting.
As we gather points in unexplored regions, we need our model to quickly adapt its predictions at the neighbouring locations of the new observations. 
If the PT is not able to do so without acquiring a significant number of observations, it will consequently impact the sample-efficiency of the optimisation procedure.
While a GP would typically be retrained on all points including new observations and adapt its parameters, 
we aim to avoid fine-tuning so to also gain computational-time advantages while keeping sample complexities (almost) unchanged compared to GP-based BO.





Figure \ref{fig:pt_growing} illustrates this behaviour with an example. We show the PT posterior distribution on Ackley, Griewank, and Rastrigin in 1 dimension, when given many observations in intervals $[0,0.15]$ and $[0.85,1]$, and relatively few in $[0.45,0.55]$. 
While growing the number of observations in the central interval, we observe a change in the predictions in that same region. 
However, that change occurs only after sufficiently many observations are available in that region, underlying the aforementioned sensitivity issue.

\section{Probabilistic transformer for Bayesian optimisation}\label{sec:pt4bo}
In light of these observations, we propose a simple yet effective training procedure focused on reliably extending the use of PTs to BO. In what comes next, we introduce two components to tackle each of the drawbacks in Section \ref{sec:drawbacks} separately.

\subsection{Non-uniformly distributed inputs prior}\label{subsec:pt4bo:non_unif_prior}
For the PT to be successful when used as a surrogate model in BO, we want to train on examples that resemble situations faced during optimisation.
As discussed in Section \ref{subsec:drawbacks:non_uniform} and illustrated in Figure \ref{fig:non_unif_probes}, throughout  optimisation, we use the model to make predictions at locations where we believe the posterior yields lower function values compared to points observed so-far\footnote{Of course, if we are interested in maximisation, we would be looking for points with higher function values.}.
Progressively, as the model observes better points and becomes more precise, it probes inputs with better values until it reaches a local optimiser.

With this in mind, we require that the PT perform good predictions especially on points with lower function values, when conditioned on data with (potentially) higher values.

In order to create such examples over which to train, we split the sampled datasets proportionally to its target values.
In particular, after having sampled a dataset $D = \{(\bm{x}, \bm{y})\}$ as in Section \ref{subsec:background:pt}, we split the data into $D_{\text{obs}}$ and $D_{\text{pred}}$ by sampling with probability $\operatorname{softmax}(\bm{y})$ such that each input-target pair $(\bm x_i,y_i) \in D$ is added in $D_{\text{obs}}$ with probability $e^{y_i} / \sum_{n=1}^{N} e^{y_n}$. 

This creates examples in which the predicted locations are distributed proportionally to the target values, i.e., non-uniformly. Additionally, this split presents the model with examples in which it has to make good posterior approximations at locations where the targets are often lower than the available conditioning points, ensuring a closer situation to that faced during BO.

\subsection{Inducing stationarity through regularisation}\label{subsec:pt4bo:reg}
As noted in Section \ref{subsec:drawbacks:sensitivity}, stationary points of PTs can sometimes increase sample complexity requiring many observations especially in novel (or exploratory) regions before performing accurate  predictions. To tackle this problem, we propose to regularise the loss function in Equation \ref{eq:ce_loss} such that we encourage desirable behaviour as we detail next. 

Given our interest in stationary and homoscedastic BO scenarios (Section \ref{sec:background_and_related_work}) - the gold standard for GP-based BO - we notice a common kernel stationarity assumption in the GP literature \citep{2006_Rasmussen, 2012_Snoek}. If augmented with that property, PTs could help in making better predictions at low sample complexities. Specifically, stationarity implies that if two input probes $\bm{x}$ and $\bm{x}^{\prime}$ are close in norm, their predictive posterior distribution is also similar -- see Appendix \ref{app:proof_rasul} for a formal derivation of this property. Not only does this property ease the search in the space of kernel hyperparameter but also induces a structure that aids GPs in generalisation \citep{2006_Rasmussen}.   

To promote a similar behaviour in PTs enabling an increased sensitivity to exploratory regions during BO, we propose to trade-off accuracy versus posterior similarity between neighbouring points through regularisation. To elaborate, consider an input point $\bm{x}_{j}^{\text{pred}}$ on which we wish to make predictions. Given $\bm{x}_{j}^{\text{pred}}$, we define a ball centered at $\bm{x}_{j}^{\text{pred}}$ of radius $\epsilon$ such that $\mathcal{B}_{\epsilon}(\bm{x}^{\text{pred}}_{j}) = \{\bm{x}^{\text{pred}}_{i}| d(\bm{x}_j^{\text{pred}},\bm{x}_{i}^{\text{pred}}) \leq \epsilon\}$ for all $i\in[1:n_{\text{pred}}]$ and $d(\cdot,\cdot)$ being a relevant distance measure. Given that the points in the ball lie within a distance $\epsilon$ from $\bm{x}_{j}^{\text{pred}}$, we regularise the parameters to yield similar posterior predictions: 
\begin{equation}
\label{eq:reg}
    \mathcal{R}(\bm{\phi}; \bm{x}_{j}^{\text{pred}}) = \sum_{\bm{x}_{i}^{\text{pred}} \in \mathcal{B}_{\epsilon}(\bm{x}_{j}^{\text{pred}})}\omega(\bm{x}_{j}^{\text{pred}}, \bm{x}_{i}^{\text{pred}}) \text{KL}(q_{\bm{\phi}}(\cdot|\bm{x}_{j}^{\text{pred}}, D_{\text{obs}})||q_{\bm{\phi}}(\cdot|\bm{x}_{i}^{\text{pred}}, D_{\text{obs}})),
\end{equation}
with $\text{KL}(p||q)$ being the Kullback–Leibler divergence between the distributions $p$ and $q$. Additionally to the $\text{KL}(\cdot||\cdot)$, Equation \ref{eq:reg} also introduces the factor $\omega(\bm{x}_{j}^{\text{pred}}, \bm{x}_{i}^{\text{pred}})$ weighing the contribution of each $\bm{x}_{i}^{\text{pred}} \in \mathcal{B}_{\epsilon}(\bm{x}_{j}^{\text{pred}})$ based on its relative distance to $\bm{x}_{j}^{\text{pred}}$ defined as follows: 
\begin{equation*}
    \omega(\bm{x}_{j}^{\text{pred}}, \bm{x}_{i}^{\text{pred}}) = \text{ReLU}(1 - d(\bm{x}_{j}^{\text{pred}}, \bm{x}_{i}^{\text{pred}})/\epsilon),
\end{equation*}
with $\text{ReLU}(a) = \max\{0,a\}$. Clearly, this term emphasises closer points to $\bm{x}_{j}^{\text{pred}}$ in the regulariser since as $\bm{x}_{i}^{\text{pred}}\rightarrow \bm{x}_{j}^{\text{pred}}$, $\omega(\bm{x}_{j}^{\text{pred}}, \bm{x}_{i}^{\text{pred}}) \rightarrow 1$. Of course, such an effective is also in agreement with stationary GP assumptions, whereby the closer the points get so does their posteriors.

The regulariser in Equation \ref{eq:reg} is added to the cross entropy loss from Equation \ref{eq:ce_loss} forming a better behaving optimisation problem that promotes more sensitivity to new observations. That is since when a new point $\bm x'$ is collected far from the current observations, the prediction around $\bm x'$ will weight $\bm x'$ more than distant observations. Of course, our regularisation introduces a new hyperparameter $\epsilon$ that is ultimately application dependent. To shed-light on its effect and provide insights that can help practitioners, we conduct an ablation study in Appendix \ref{app:ablation}.

\section{Bayesian optimisation experiments and results}\label{sec:exp}
To assess the quality of our approach, we focus our experimental analysis on standard BO benchmarks and consider 16 functions implemented in the \texttt{BoTorch} \citep{2020_Balandat} suite. 
These well-known objective functions are commonly used in the BO literature  \citep{2012_Snoek,2016_Springenberg,2021_Shangguan,2021_Astudillo_a,2021_Astudillo_b,2021_Moss,2021_Eriksson,2020_Letham,2020_Jiang} as they provide a large, diverse and challenging source of optimisation problems. 
Since GP-based BO is the gold standard and is usually still the best algorithm across all tasks (see Section \ref{sec:exp}) from this panel, we compare our regret results to that classical benchmark. It is worth noting that we also assess the importance of each of the contributions in Section \ref{sec:pt4bo}separately on the surrogate's performance during the BO runs across various black-boxes and dimensions. Our results indicate significant accuracy improvements when compared to standard PTs. Due to space constraints, however, we report some of those findings in Figure \ref{fig:results}, the rest being displayed in Appendix \ref{app:visualExps}.  

\paragraph{Experimental details.}
We run four sets of experiments in 1, 2, 5, and 10 dimensions.
For each dimension we train a variation of the PT to study the contribution of each of the introduced components (see Table \ref{tab:notation}).
In all dimensions, we train the original PT, the PT with the non-uniform prior, the PT with regularisation and the PT with both components.
We then use this \emph{one} pre-trained model on all black-boxes in that dimension and report simple regret results.
\begin{table}[htp]
\caption{Algorithms notation.}
\label{tab:notation}
    \centering
    \begin{tabular}{l|l}
        Notation & Components \\
        \hline
        RS & Random Search \\
        GP & Gaussian Process BO \\
        PT & \citep{2021_Muller} unchanged \\
        PT-$\mathcal{R}_{\varepsilon}$ & PT with Regularised Training (Equation \ref{eq:reg}) \\
        PT-$\nu$ & PT with Non-Unif. Inputs Prior \\
        PT-$\nu\mathcal{R}_{\varepsilon}$ & PT with Non-Unif. Inputs Prior \& with Regularised Training \\
    \end{tabular}
\end{table}

\begin{figure}[htp]
    \centering
    \includegraphics[width=1.0\textwidth]{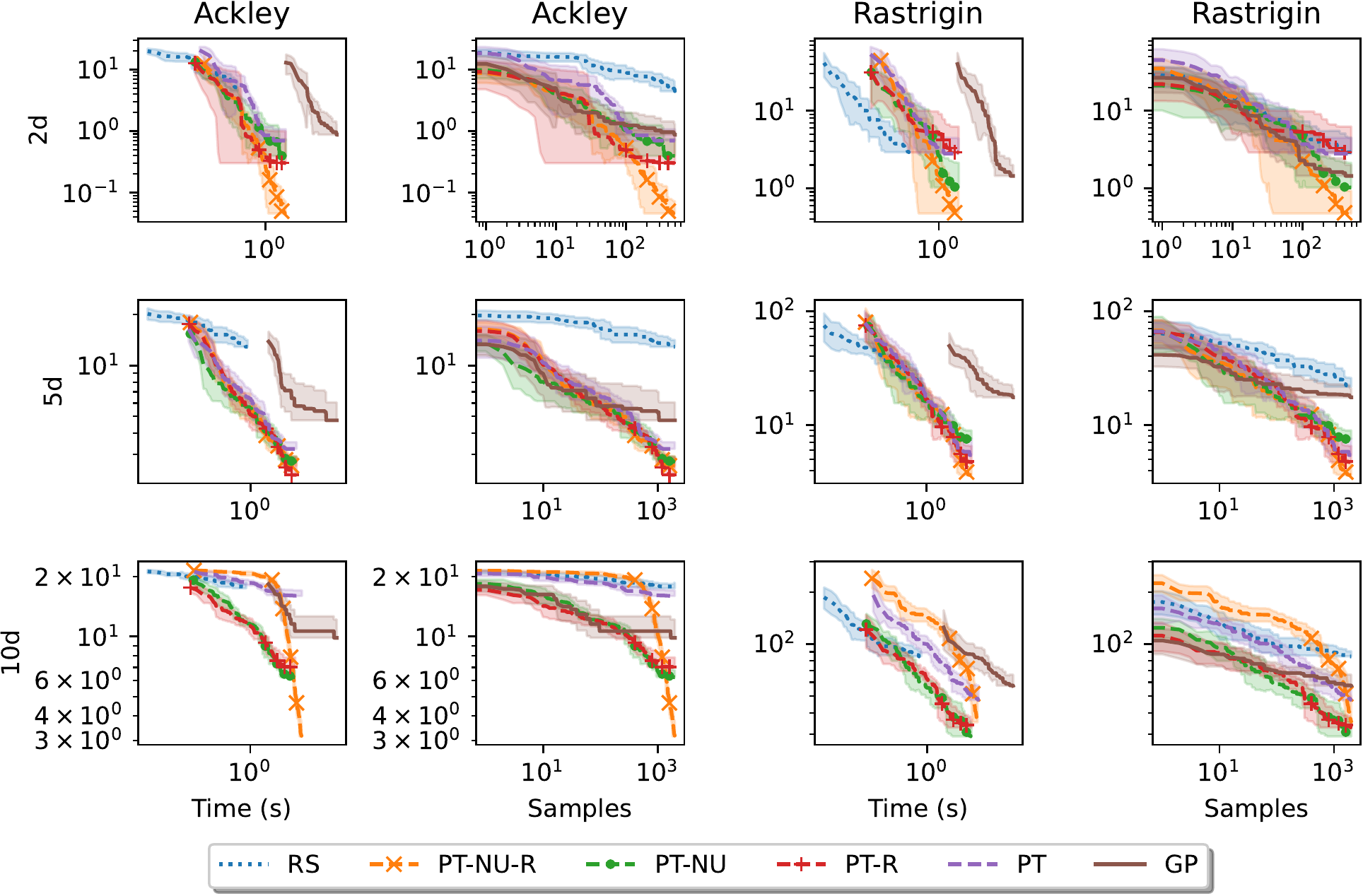}
    \caption{Simple Regret on Ackley and Rastrigin in 2, 5 and 10 dimensions for RS, PT, PT-$\nu$, PT-$\nu\mathcal{R}_{\varepsilon}$ and GP. Results are shown in log-scale and versus wall-clock time (first and third column) as well as versus iteration numbers (second and fourth column).}
    \label{fig:results}
\end{figure}

To run the BO for a specific objective and model, we start by sampling $n_{\text{init}}$ initial points uniformly - constituting an initial dataset on which to condition the predictions.
For all experiments, we use Expected Improvement (EI) as our acquisition function.
Finally, we report the results in terms of simple regret and run each experiment 10 times for statistical significance.
We detail all hyperparameters related to the model's training (such as learning rate, $\varepsilon$, number of epochs, optimiser, and so-forth) or the BO loop's settings (e.g., the number of steps, $n_{\text{init}}$, etc.) in Appendix \ref{app:add_exp}.

In Figure \ref{fig:results}, we display the simple regret results of the optimisation experiments in 2, 5 and 10 dimensions on arguably the most difficult objective functions of the panel: Ackley and Rastrigin.
For each setting, we show average regret throughout optimisation with 0.5 standard deviation across 10 seeds.
Although we ran experiments on 16 different functions, we showcase those three here as they have been our running examples throughout this work and are arguably among the most difficult objectives of the panel.
Figure \ref{fig:results} shows the same regret results for each task when plotting against the number of function acquisitions (i.e. black-box queries) and against wall-clock time.
This figure clearly shows that the PT trained with our added components is not only a competitive surrogate model but is also an order of magnitude faster than GP-based BO.
Additional time comparison and PT training times are discussed in Appendix \ref{app:time}.

\paragraph{Average results on all benchmarks.} We also report average results on all experiments in Table \ref{tab:average_results}.
A standard way of comparing results across tasks with different scales \citep{2021_Turner} is to transform absolute regret in relative performance to a baseline, i.e. Random Search.
Hence, we report percentage improvement over Random Search for each method, broken down by dimension and averaged over different tasks.
For each dimension, we evaluate each method on all the available tasks counted in the last column of Table \ref{tab:average_results}.
Note that when a method does not perform better than Random Search (which happens for the PT as well as the GP on some tasks) we cap at 0\% improvement.
Table \ref{tab:average_results} reveals that the components introduced in Section \ref{sec:pt4bo} can drastically improve the performance of the BO using the PT as a model. 
Importantly, they also confirm that the proposed contributions allow PT-$\nu\mathcal{R}_{\varepsilon}$ to be a very competitive alternative to GPs in BO.

\begin{table}[ht!]
\caption{Average performance measured by \% improvement over RS, per model and by dimensions.}
\label{tab:average_results}
    \centering
    \begin{tabular}{c|cccc|c|c}
        \hline
        Dimension & PT & PT-$\mathcal{R}_{\varepsilon}$ & PT-$\nu$ & PT-$\nu\mathcal{R}_{\varepsilon}$ & GP & \# Tasks \\
        \hline
        1 & 20.5 & 88.1 & 79.2 & 73.6 & \textbf{93.6} & 6 \\
        2 & 103.3 & 112.4 & 111.8 & \textbf{114.3} & 111 & 16 \\
        5 & 73 & 77 & 74.7 & \textbf{78} &  65.8 & 8 \\
        10 & 36 & 71.1 & 71.8 & \textbf{72.1} & 63.8 & 8 \\
        \hline
    \end{tabular}
\end{table}


\section{Conclusions, limitations, and future work}
\label{sec:conclusion}
We propose the use of a recently developed transformer architecture as a surrogate model for black-box optimisation.
Nevertheless, its direct application exhibits undesirable properties which we analyse.
To mitigate those shortcomings, we introduce two components: 1) a non-uniform prior and 2) a novel regularisation term. 
Given the same pre-trained model on data sampled from random hyper-GP priors and using only forward passes (no fine-tuning on the true black-boxes), we achieve an order of magnitude time reduction compared to GP-based BO while not increasing sample complexities. To our knowledge, this is among the first works demonstrating sample efficient Bayesian optimisation with transformer models.

\paragraph{Limitations.} Although successful in those instances, the limitations of our work mostly lie in the pre-training of the PT. Notably in higher dimensions, we require more samples for a better predictive posteriors. Unfortunately, the memory requirement of our model is quadratic in terms of the number of data points. This, of course, can prohibit scaling our results into large dimensional problems. 
Another limitation that is shared with most machine learning algorithms in that our model requires tuning hyperparameters common to BO, while introducing yet another hyperparameter, $\epsilon$, needed by our regulariser. Scaling our probabilistic transformer is an orthogonal but important direction for future work as noted below. 
Concerning the choice of $\varepsilon$, we presented an ablation study in Appendix \ref{app:hyperparam}.  

\paragraph{Future work.} There are many exciting avenues for future research. First, we plan to scale our model to higher-dimensional settings by proposing novel distributed computing schemes. Second, we wish to consider difficult BO problems involving heteroscedastic and non-smooth black-boxes.


\bibliographystyle{unsrtnat}
\bibliography{bib}
\newpage

\section{Appendix}
In this appendix we wish to provide insights that we did not have space to detail in the main paper. 
In particular we give a proof of our claim from Section \ref{subsec:pt4bo:reg} first in the Section \ref{app:proof_rasul}.
Moreover, we provide an ablation experiment on the regulariser hyperparameter $\varepsilon$ in Section \ref{app:ablation} and give an intuitive illustration of the introduced components in Section \ref{app:visualExps}.
Finally, we provide results for all experiments in Section \ref{app:add_exp}, hyperparameters in Section \ref{app:hyperparam} and report time and hardware details in Section \ref{app:time}.

\subsection{Derivation}
\label{app:proof_rasul}
While introducing our novel regulariser in the main paper, we noted that under most common stationary kernels, GP posterior distributions become closer as inputs get closer. In this section of the appendix, we provide a formal argument supporting our claim. 

We start with assumptions on the black-box objective function and the kernel of the associated surrogate GP model.

\begin{assumption}\label{A_1}
A black-box objective function $f(\cdot)$ is defined over a bounded domain $\mathcal{X} \subset \mathbb{R}^d$ of diameter $R$ and its noisy evaluations $y(\bm{x}) = f(\bm{x}) + \mathcal{N}(0,\sigma^2_{\text{noise}})$ are bounded, i.e.
\begin{align*}
    \sup_{(\bm{x},\tilde{\bm{x}})\in\mathcal{X}^2}||\bm{x} - \tilde{\bm{x}}||_2 \le R,\  \ \ |y(\bm{x})| \le C, \ \ \ \forall \bm{x}\in\mathcal{X} \ \ \ \  \text{ with probability }1.
\end{align*}
\end{assumption}

\begin{assumption}\label{A_2}
The Gaussian process surrogate associated with function $f(\cdot)$ has a covariance kernel $k_{\bm{\theta}}(\cdot, \cdot):\mathcal{X}\times\mathcal{X}\to \mathbb{R}^{+}$, such that for all $\bm{x},\bm{z}\in\mathcal{X}$ and all $\bm{\theta}$
\begin{align*}
    &k_{\bm{\theta}}(\bm{x},\bm{x}) = c,\\\nonumber
    &||\nabla_{\bm{x}}k_{\bm{\theta}}(\bm{x}, \bm{z})||_2 \le B,\\\nonumber
    &k_{\bm{\theta}}(\bm{x},\bm{z})\ge b.
\end{align*}
\end{assumption}

Please notice that the above assumptions are standard in the literature \citep{2006_Rasmussen} and hold for popular stationary kernels (e.g., squared exponential and Mat\'ern kernels) defined over the bounded search region $\mathcal{X}$.

Given these assumptions, we are ready to establish the proximity result between posterior distributions associated with two close input locations:

\begin{lemma}
Let $\bm{x},\tilde{\bm{x}}\in\mathcal{X}$ such that $||\bm{x} - \tilde{\bm{x}}||_2\le \eta$ for some small positive $\eta$, and let us consider two GP posterior distributions $\mathcal{N}(\mu(\bm{x}|D), \sigma^2(\bm{x}|D))$ and $\mathcal{N}( \mu(\tilde{\bm{x}}|D), \sigma^2(\tilde{\bm{x}}|D)$ modelling function $f(\cdot)$ at input locations $\bm{x}$ and $\tilde{\bm{x}}$ respectively after observing dataset $D = \{\bm{x}_i, y_i\}^N_{i=1}$. Then, with probability 1 we have
\begin{align*}
    \text{KL}\left(\mathcal{N}( \mu(\bm{x}|D), \sigma^2(\bm{x}|D)) || \mathcal{N}( \mu(\tilde{\bm{x}}|D), \sigma^2(\tilde{\bm{x}}|D)\right) \le \tilde{C}\eta.
\end{align*}
for some positive constant $\tilde{C}$. 
\end{lemma}

\begin{proof} Consider two close points $\bm{x},\tilde{\bm{x}}\in\mathcal{X}$ such that $||\bm{x} - \tilde{\bm{x}}||_2\le \eta$ for some small value $\eta > 0$. Then, for the posterior means we have
\begin{align*}
    &\mu(\bm{x}|D) = k^{\mathsf{T}}_{\bm{\theta}}(\bm{x}, D)\left[\text{K}_{\bm{\theta}}(D) + \sigma^2_{\text{noise}}\text{I}\right]^{-1}\bm{y}, \\\nonumber
    &\mu(\tilde{\bm{x}}|D) = k^{\mathsf{T}}_{\bm{\theta}}(\tilde{\bm{x}},D)\left[\text{K}_{\bm{\theta}}(D) + \sigma^2_{\text{noise}}\text{I}\right]^{-1}\bm{y}.
\end{align*}
where $\bm{y} = [y_i]^N_{i=1}$ is a vector of function observations, and $\text{K}_{\bm{\theta}}(D) = [k_{\bm{\theta}}(\bm{x}_i, \bm{x}_j)]_{\bm{x}_i, \bm{x}_j \in D}$ is a covariance matrix associated with observations in $D$. For simplicity, let us denote $\text{A} = \left[\text{K}_{\bm{\theta}}(D) + \sigma^2_{\text{noise}}\text{I}\right]^{-1}$. Then, using Assumption \ref{A_1}, we have
\begin{align*}
    &|\mu(\bm{x}|D) - \mu(\tilde{\bm{x}}|D)| = |k^{\mathsf{T}}_{\bm{\theta}}(\bm{x}, D)\text{A}\bm{y} - k^{\mathsf{T}}_{\bm{\theta}}(\tilde{\bm{x}}, D)\text{A}\bm{y}| \le \\\nonumber 
    &||k_{\bm{\theta}}(\bm{x}, D) - k_{\bm{\theta}}(\tilde{\bm{x}}, D)||_2||\text{A}||_2||\bm{y}||_2 \le C||k_{\bm{\theta}}(\bm{x}, D) - k_{\bm{\theta}}(\tilde{\bm{x}}, D)||_2||\text{A}||_2, \ \ \ \ \text{w.p. 1}.
\end{align*}
Next, notice that $\text{A} = \text{K}_{\bm{\theta}}(D) + \sigma^2_{\text{noise}}\text{I} \succeq \sigma^2_{\text{noise}}\text{I}$, hence $\text{A} \preceq \frac{1}{\sigma^2_{\text{noise}}}\text{I}$ and 
\begin{align}\label{result_poster}
    |\mu(\bm{x}|D) - \mu(\tilde{\bm{x}}|D)| \le  \frac{C}{\sigma^2_{\text{noise}}}||k_{\bm{\theta}}(\bm{x}, D) - k_{\bm{\theta}}(\tilde{\bm{x}}, D)||_2, \ \ \ \ \text{w.p. 1.}
\end{align}
Let us bound the term $||k_{\bm{\theta}}(\bm{x}, D) - k_{\bm{\theta}}(\tilde{\bm{x}}, D)||_2$. By definition, we have
\begin{align}\label{interm_result_one}
    &||k_{\bm{\theta}}(\bm{x}, D) - k_{\bm{\theta}}(\tilde{\bm{x}}, D)||_2 = \sqrt{\sum_{j=1}^N\left(k_{\bm{\theta}}(\bm{x}, \bm{x}_j) - k_{\bm{\theta}}(\tilde{\bm{x}}, \bm{x}_j)\right)^2} \le \\\nonumber
    &\sum_{j=1}^N|k_{\bm{\theta}}(\bm{x}, \bm{x}_j) - k_{\bm{\theta}}(\tilde{\bm{x}}, \bm{x}_j)| = \sum_{j=1}^Nk_{\bm{\theta}}(\bm{x}, \bm{x}_j)\left|1 - \frac{k_{\bm{\theta}}(\tilde{\bm{x}}, \bm{x}_j)}{k_{\bm{\theta}}(\bm{x}, \bm{x}_j)}\right|\le c\sum_{j=1}^N\left|1 - \frac{k_{\bm{\theta}}(\tilde{\bm{x}}, \bm{x}_j)}{k_{\bm{\theta}}(\bm{x}, \bm{x}_j)}\right|.
\end{align} 
where in the last step we use Assumption \ref{A_2} (specifically, for any $\bm{x},\bm{z}\in\mathcal{X}$ we have $k_{\bm{\theta}}(\bm{x},\bm{z}) \le k_{\bm{\theta}}(\bm{x}, \bm{x}) = c$). Next, consider the Taylor expansion of the kernel
\begin{align*}
    k_{\bm{\theta}}(\tilde{\bm{x}}, \bm{x}_j) = k_{\bm{\theta}}(\bm{x}, \bm{x}_j) + \nabla_{\bm{x}}k^{\mathsf{T}}_{\bm{\theta}}(\bm{x}, \bm{x}_j)\left.\right|_{\bm{x} = \tilde{\bm{x}}} [\tilde{\bm{x}} - \bm{x}] + o(||\tilde{\bm{x}} - \bm{x}||_2).
\end{align*}
Therefore, using Assumption \ref{A_2} we have
\begin{align*}
    &\frac{k_{\bm{\theta}}(\tilde{\bm{x}}, \bm{x}_j)}{k_{\bm{\theta}}(\bm{x}, \bm{x}_j)} = 1 + \frac{\nabla_{\bm{x}}k^{\mathsf{T}}_{\bm{\theta}}(\bm{x}, \bm{x}_j)\left.\right|_{\bm{x} = \tilde{\bm{x}}}[\tilde{\bm{x}} - \bm{x}]}{k_{\bm{\theta}}(\bm{x}, \bm{x}_j)} + \frac{o(\eta)}{k_{\bm{\theta}}(\bm{x}, \bm{x}_j)} = \\\nonumber
    &1 + \frac{\nabla_{\bm{x}}k^{\mathsf{T}}_{\bm{\theta}}(\bm{x}, \bm{x}_j)\left.\right|_{\bm{x} = \tilde{\bm{x}}}[\tilde{\bm{x}} - \bm{x}]}{k_{\bm{\theta}}(\bm{x}, \bm{x}_j)} + o(\eta).
\end{align*}
Hence, using Assumption \ref{A_2} and the Cauchy–Schwarz inequality we obtain
\begin{align*}
    &\left|1 - \frac{k_{\bm{\theta}}(\tilde{\bm{x}}, \bm{x}_j)}{k_{\bm{\theta}}(\bm{x}, \bm{x}_j)} \right| \le \left|\frac{\nabla_{\bm{x}}k^{\mathsf{T}}_{\bm{\theta}}(\bm{x}, \bm{x}_j)\left.\right|_{\bm{x} = \tilde{\bm{x}}}[\tilde{\bm{x}} - \bm{x}]}{k_{\bm{\theta}}(\bm{x}, \bm{x}_j)} + o(\eta)\right| \le \\\nonumber
    &\frac{||\nabla_{\bm{x}}k^{\mathsf{T}}_{\bm{\theta}}(\bm{x}, \bm{x}_j)\left.\right|_{\bm{x} = \tilde{\bm{x}}}||_2||\tilde{\bm{x}} - \bm{x}||_2}{b} + o(\eta) \le \frac{B\eta}{b} + o(\eta).
\end{align*}
Applying this result in (\ref{interm_result_one}) gives
\begin{align}\label{interm_result_two}
    &||k_{\bm{\theta}}(\bm{x}, D) - k_{\bm{\theta}}(\tilde{\bm{x}}, D)||_2 \le c\sum_{j=1}^N\left[\frac{B\eta}{b}+ o(\eta)\right] = \frac{cNB}{b}\eta + o(\eta).
\end{align}
Hence, for the posterior mean difference in (\ref{result_poster}) we have
\begin{align}\label{posterior_final_bound}
    &|\mu(\bm{x}|D) - \mu(\tilde{\bm{x}}|D)| \le \frac{C}{\sigma^2_{\text{noise}}}\left[\frac{cNB}{b}\eta + o(\eta)\right]\ \ \ \ \text{w.p. 1.}
\end{align}
Next, let us bound the difference between the posterior variances. Following closed form expressions for posterior variances and Assumption \ref{A_2} we have that
\begin{align*}
    &\sigma^2(\bm{x}|D) = c - k^{\mathsf{T}}_{\bm{\theta}}(\bm{x}, D)\text{A}k_{\bm{\theta}}(\bm{x}, D),\\\nonumber
    &\sigma^2(\tilde{\bm{x}}|D) = c - k^{\mathsf{T}}_{\bm{\theta}}(\tilde{\bm{x}}, D)\text{A}k_{\bm{\theta}}(\tilde{\bm{x}}, D).
\end{align*}
Therefore, using Assumption \ref{A_2} we can write
\begin{align*}
    &|\sigma^2(\bm{x}|D) - \sigma^2(\tilde{\bm{x}}|D)| = |k^{\mathsf{T}}_{\bm{\theta}}(\bm{x}, D)\text{A}k_{\bm{\theta}}(\bm{x}, D) - k^{\mathsf{T}}_{\bm{\theta}}(\tilde{\bm{x}}, D)\text{A}k_{\bm{\theta}}(\tilde{\bm{x}}, D)| = \\\nonumber
    &|k^{\mathsf{T}}_{\bm{\theta}}(\bm{x}, D)\text{A}k_{\bm{\theta}}(\bm{x}, D) - k^{\mathsf{T}}_{\bm{\theta}}(\bm{x}, D)\text{A}k_{\bm{\theta}}(\tilde{\bm{x}}, D) + k^{\mathsf{T}}_{\bm{\theta}}(\bm{x}, D)\text{A}k_{\bm{\theta}}(\tilde{\bm{x}}, D) - k^{\mathsf{T}}_{\bm{\theta}}(\tilde{\bm{x}}, D)\text{A}k_{\bm{\theta}}(\tilde{\bm{x}}, D)| \le \\\nonumber
    &||k_{\bm{\theta}}(\bm{x}, D) - k_{\bm{\theta}}(\tilde{\bm{x}}, D)||_2||\text{A}||_2||k_{\bm{\theta}}(\bm{x},D)||_2 + ||k_{\bm{\theta}}(\bm{x}, D) - k_{\bm{\theta}}(\tilde{\bm{x}}, D)||_2||\text{A}||_2||k_{\bm{\theta}}(\tilde{\bm{x}},D)||_2 \le \\\nonumber
    &\frac{2c\sqrt{N}}{\sigma^2_{\text{noise}}}||k_{\bm{\theta}}(\bm{x}, D) - k_{\bm{\theta}}(\tilde{\bm{x}}, D)||_2.
\end{align*}
Applying bound (\ref{interm_result_two}) in the above result gives
\begin{align}\label{posterior_variance_result}
    |\sigma^2(\bm{x}|D) - \sigma^2(\tilde{\bm{x}}|D)| \le \frac{2c\sqrt{N}}{\sigma^2_{\text{noise}}}\left[\frac{cNB}{b}\eta + o(\eta)\right].
\end{align}
Now, for the KL divergence between two univariate Gaussian distributions $\mathcal{N}(\mu(\bm{x}|D), \sigma^2(\bm{x}|D))$ and $\mathcal{N}( \mu(\tilde{\bm{x}}|D), \sigma^2(\tilde{\bm{x}}|D)$ with non-trivial variances $\min\{\sigma^2(\tilde{\bm{x}}|D), \sigma^2(\bm{x}|D)\} \ge \varrho > 0$ we have
\begin{align*}
    &\text{KL}(\mathcal{N}( \mu(\bm{x}|D), \sigma^2(\bm{x}|D)) || \mathcal{N}( \mu(\tilde{\bm{x}}|D), \sigma^2(\tilde{\bm{x}}|D)) = \\\nonumber
    &\log\frac{\sigma(\tilde{\bm{x}}|D)}{\sigma(\bm{x}|D)} + \frac{\sigma^2(\bm{x}|D) + (\mu(\bm{x}|D) - \mu(\tilde{\bm{x}}|D))^2}{2\sigma^2(\tilde{\bm{x}}|D)} - \frac{1}{2}.
\end{align*}
Using properties of log function, the above expression can be written as:
\begin{align*}
    &\text{KL}(\mathcal{N}( \mu(\bm{x}|D), \sigma^2(\bm{x}|D)) || \mathcal{N}( \mu(\tilde{\bm{x}}|D), \sigma^2(\tilde{\bm{x}}|D)) = \\\nonumber &\frac{1}{2}\log\frac{\sigma^2(\tilde{\bm{x}}|D)}{\sigma^2(\bm{x}|D)} + \frac{\sigma^2(\bm{x}|D) + (\mu(\bm{x}|D) - \mu(\tilde{\bm{x}}|D))^2}{2\sigma^2(\tilde{\bm{x}}|D)} - \frac{1}{2}.
\end{align*}
and, after simplification:
\begin{align}\label{KL_result_one}
    &\text{KL}(\mathcal{N}( \mu(\bm{x}|D), \sigma^2(\bm{x}|D)) || \mathcal{N}( \mu(\tilde{\bm{x}}|D), \sigma^2(\tilde{\bm{x}}|D)) = \\\nonumber &\frac{1}{2}\log\frac{\sigma^2(\tilde{\bm{x}}|D)}{\sigma^2(\bm{x}|D)} + \frac{\sigma^2(\bm{x}|D) + (\mu(\bm{x}|D) - \mu(\tilde{\bm{x}}|D))^2}{2\sigma^2(\tilde{\bm{x}}|D)} - \frac{1}{2} = \\\nonumber
    &\frac{1}{2}\log\left[1 + \frac{\sigma^2(\tilde{\bm{x}}|D) -\sigma^2(\bm{x}|D) }{\sigma^2(\bm{x}|D)}\right] + \frac{\sigma^2(\bm{x}|D) - \sigma^2(\tilde{\bm{x}}|D) + (\mu(\bm{x}|D) - \mu(\tilde{\bm{x}}|D))^2}{2\sigma^2(\tilde{\bm{x}}|D)}.
\end{align}
Using result (\ref{posterior_variance_result}), can have
\begin{align*}
    &\log\left[1 + \frac{\sigma^2(\tilde{\bm{x}}|D) -\sigma^2(\bm{x}|D) }{\sigma^2(\bm{x}|D)}\right] \le \frac{|\sigma^2(\tilde{\bm{x}}|D) -\sigma^2(\bm{x}|D)| }{\sigma^2(\bm{x}|D)} + o\left(\frac{|\sigma^2(\tilde{\bm{x}}|D) -\sigma^2(\bm{x}|D)| }{\sigma^2(\bm{x}|D)}\right) \le \\\nonumber
    &\frac{2c\sqrt{N}}{\sigma^2_{\text{noise}}\varrho}\left[\frac{cNB}{b}\eta + o(\eta)\right] + o\left(\frac{2c\sqrt{N}}{\sigma^2_{\text{noise}}\varrho}\left[\frac{cNB}{b}\eta + o(\eta)\right]\right) = \frac{2c^2N^{\frac{3}{2}}B}{\sigma^2_{\text{noise}}b\varrho}\eta + o(\eta).
\end{align*}
and, using results (\ref{posterior_final_bound}) and (\ref{posterior_variance_result}) we can further write
\begin{align*}
    &\frac{\sigma^2(\bm{x}|D) - \sigma^2(\tilde{\bm{x}}|D) + (\mu(\bm{x}|D) - \mu(\tilde{\bm{x}}|D))^2}{2\sigma^2(\tilde{\bm{x}}|D)} \le \frac{|\sigma^2(\bm{x}|D) - \sigma^2(\tilde{\bm{x}}|D)|}{2\sigma^2(\tilde{\bm{x}}|D)} + \frac{(\mu(\bm{x}|D) - \mu(\tilde{\bm{x}}|D))^2}{2\sigma^2(\tilde{\bm{x}}|D)} \le \\\nonumber
    &\frac{c\sqrt{N}}{\sigma^2_{\text{noise}}\varrho}\left[\frac{cNB}{b}\eta + o(\eta)\right] + \frac{C^2}{2\sigma^2_{\text{noise}}\varrho}\left[\frac{cNB}{b}\eta + o(\eta)\right]^2 = \frac{c^2N^{\frac{3}{2}}B}{\sigma^2_{\text{noise}}b\varrho}\eta + o(\eta) \ \ \ \ \text{w.p. 1.}
\end{align*}
Therefore, applying these results in (\ref{KL_result_one}) we get
\begin{align*}
    &\text{KL}(\mathcal{N}( \mu(\bm{x}|D), \sigma^2(\bm{x}|D)) || \mathcal{N}( \mu(\tilde{\bm{x}}|D), \sigma^2(\tilde{\bm{x}}|D)) = \\\nonumber
    &\frac{1}{2}\log\left[1 + \frac{\sigma^2(\tilde{\bm{x}}|D) -\sigma^2(\bm{x}|D) }{\sigma^2(\bm{x}|D)}\right] + \frac{\sigma^2(\bm{x}|D) - \sigma^2(\tilde{\bm{x}}|D) + (\mu(\bm{x}|D) - \mu(\tilde{\bm{x}}|D))^2}{2\sigma^2(\tilde{\bm{x}}|D)} \le \\\nonumber
    &\frac{c^2N^{\frac{3}{2}}B}{\sigma^2_{\text{noise}}b\varrho}\eta + o(\eta) + \frac{c^2N^{\frac{3}{2}}B}{\sigma^2_{\text{noise}}b\varrho}\eta + o(\eta) = \frac{2c^2N^{\frac{3}{2}}B}{\sigma^2_{\text{noise}}b\varrho}\eta + o(\eta) = \tilde{C}\eta.
\end{align*}
with probability 1. This finishes the proof of the lemma.
\end{proof}

\subsection{Additional experiments}\label{app:add_exp}
As mentioned in the main paper, we present in Figures \ref{fig:bo_all_runs_1d}, \ref{fig:bo_all_runs_2d}, \ref{fig:bo_all_runs_5d} and \ref{fig:bo_all_runs_10d}, the results of all the experiments we ran. 
We display the simple regret results over 10 seeds for each method (see Table \ref{tab:notation}) against the wall-clock time. 
As the average results in Table \ref{tab:average_results} show, it is very difficult to outperform the GP in 1 dimension.
In higher dimensions, however, the PT trained with our components is very competitive.
For completeness, we also show, for the same experiments, the evolution of simple regrets against the number of samples in Figures \ref{fig:bo_all_runs_1d_sample}, \ref{fig:bo_all_runs_2d_sample}, \ref{fig:bo_all_runs_5d_sample} and \ref{fig:bo_all_runs_10d_sample}.
\begin{figure}[H]
    \centering
    \includegraphics[width=0.9\textwidth]{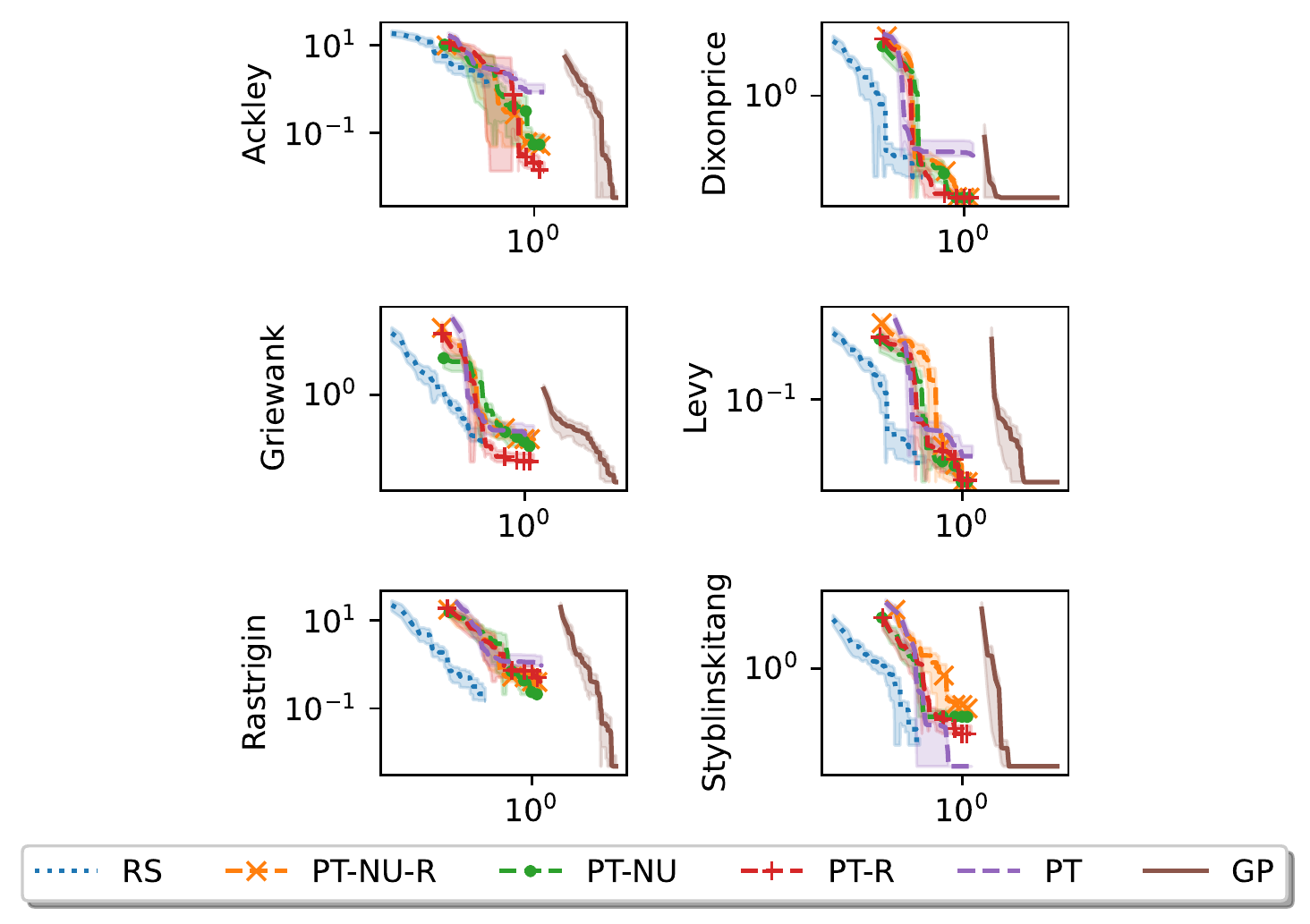}
    \caption{Simple regret results versus wall-clock time in 1 dimension. The GP is extremely sample-efficient in 1 dimension hence being hard to outperform. However, we note that our components still improve the original PT.}
    \label{fig:bo_all_runs_1d}
\end{figure}

\begin{figure}[H]
    \centering
    \includegraphics[width=0.9\textwidth]{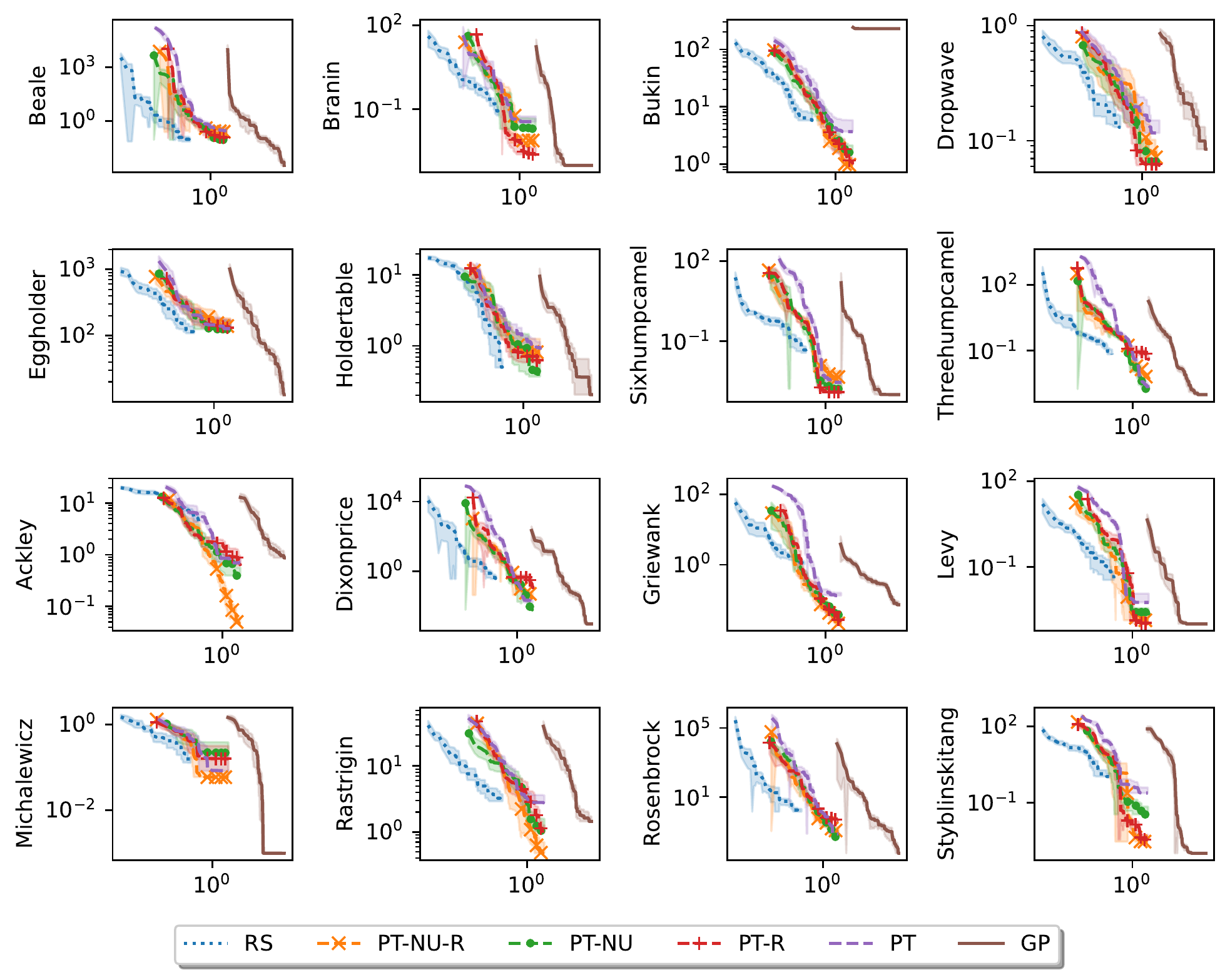}
    \caption{Simple regret results versus wall-clock time in 2 dimensions. Excluding Bukin function, the GP is still an extremely efficient surrogate but PT-$\nu\mathcal{R}_{\varepsilon}$ is able to outperform it in some cases, on top of improving over PT.}
    \label{fig:bo_all_runs_2d}
\end{figure}

\begin{figure}[H]
    \centering
    \includegraphics[width=0.8\textwidth]{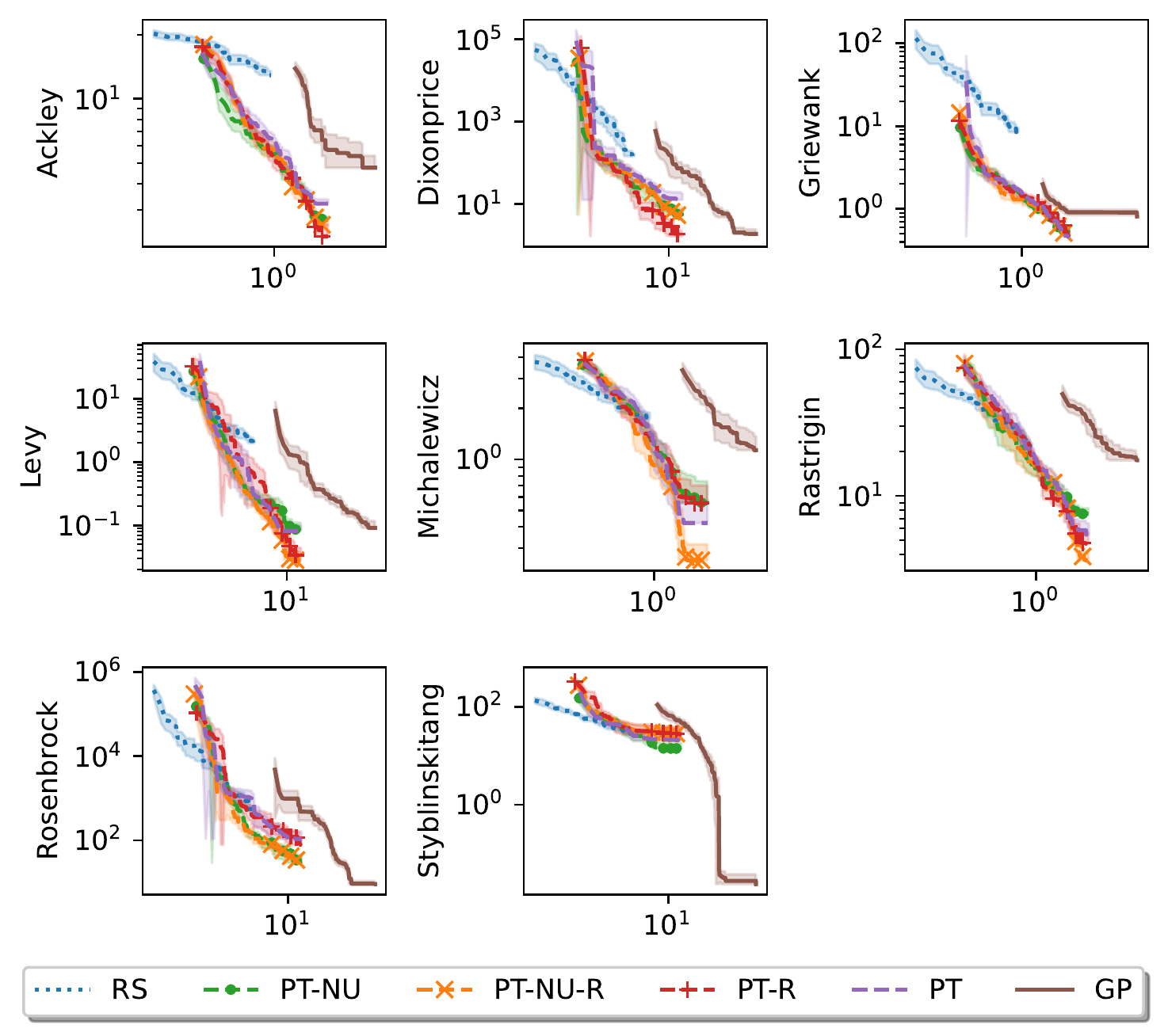}
    \caption{Simple regret results versus wall-clock time in 5 dimensions. Here, we see the advantage of our components as the PT trained with our components starts to drastically improve over the original PT and outperforms GP-based BO runs on even more tasks.}
    \label{fig:bo_all_runs_5d}
\end{figure}

\begin{figure}[H]
    \centering
    \includegraphics[width=0.8\textwidth]{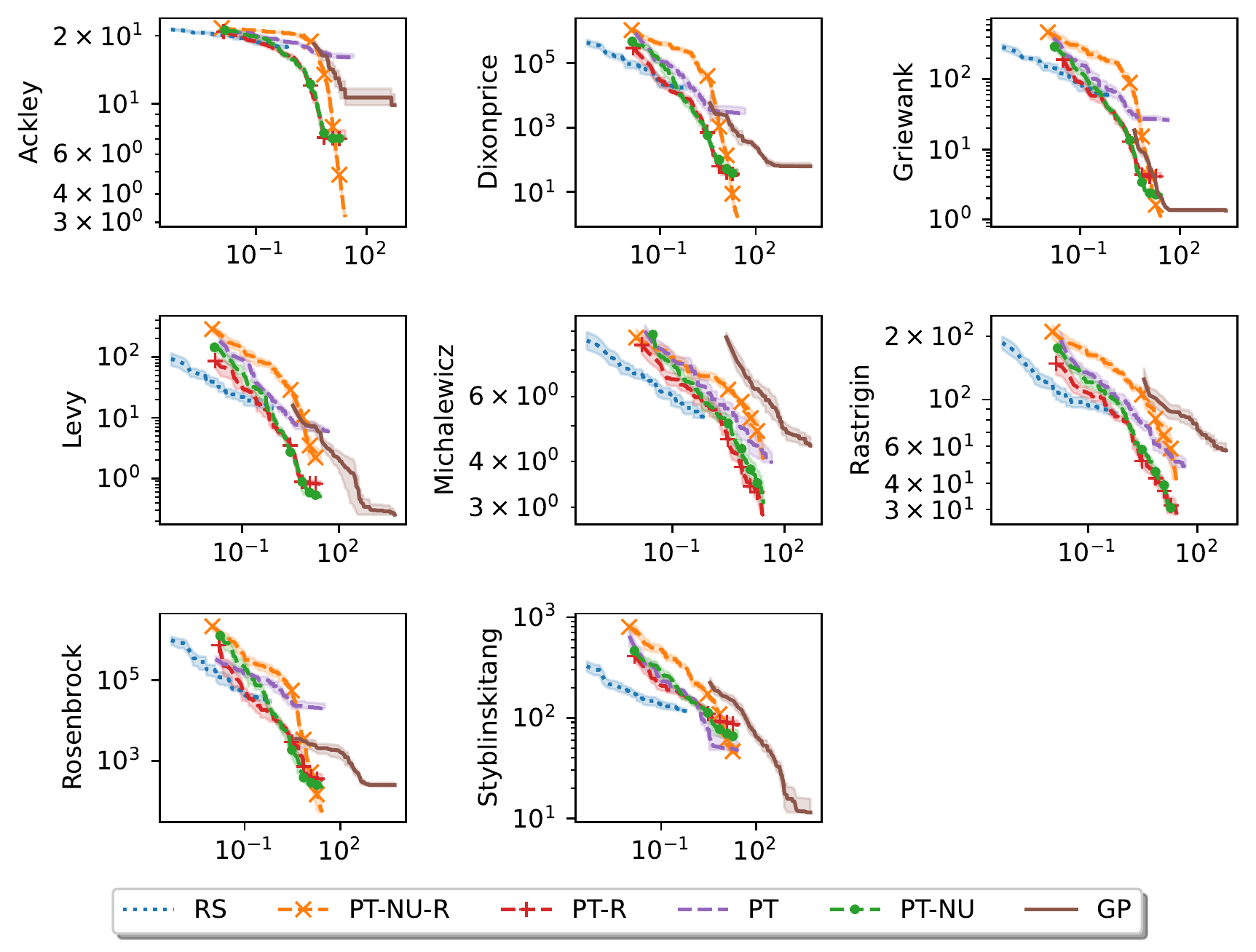}
    \caption{Simple regret results versus wall-clock time in 10 dimensions. We observe that our components systematically improve the regret as the original PT struggles even to outperform Random Search on several tasks.}
    \label{fig:bo_all_runs_10d}
\end{figure}

\begin{figure}[H]
    \centering
    \includegraphics[width=0.9\textwidth]{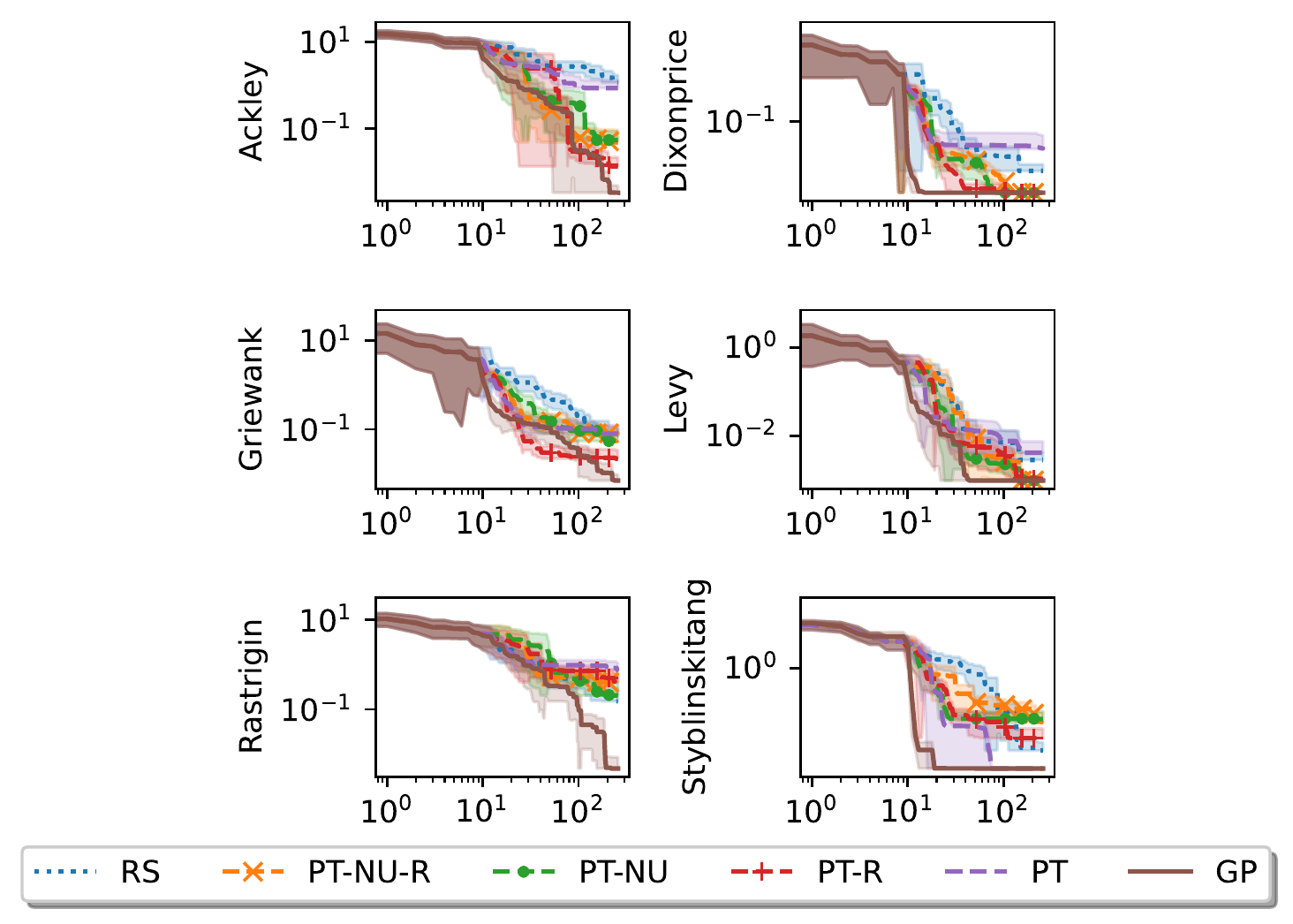}
    \caption{Simple regret results versus the number of samples time in 1 dimension.}
    \label{fig:bo_all_runs_1d_sample}
\end{figure}

\begin{figure}[H]
    \centering
    \includegraphics[width=0.99\textwidth]{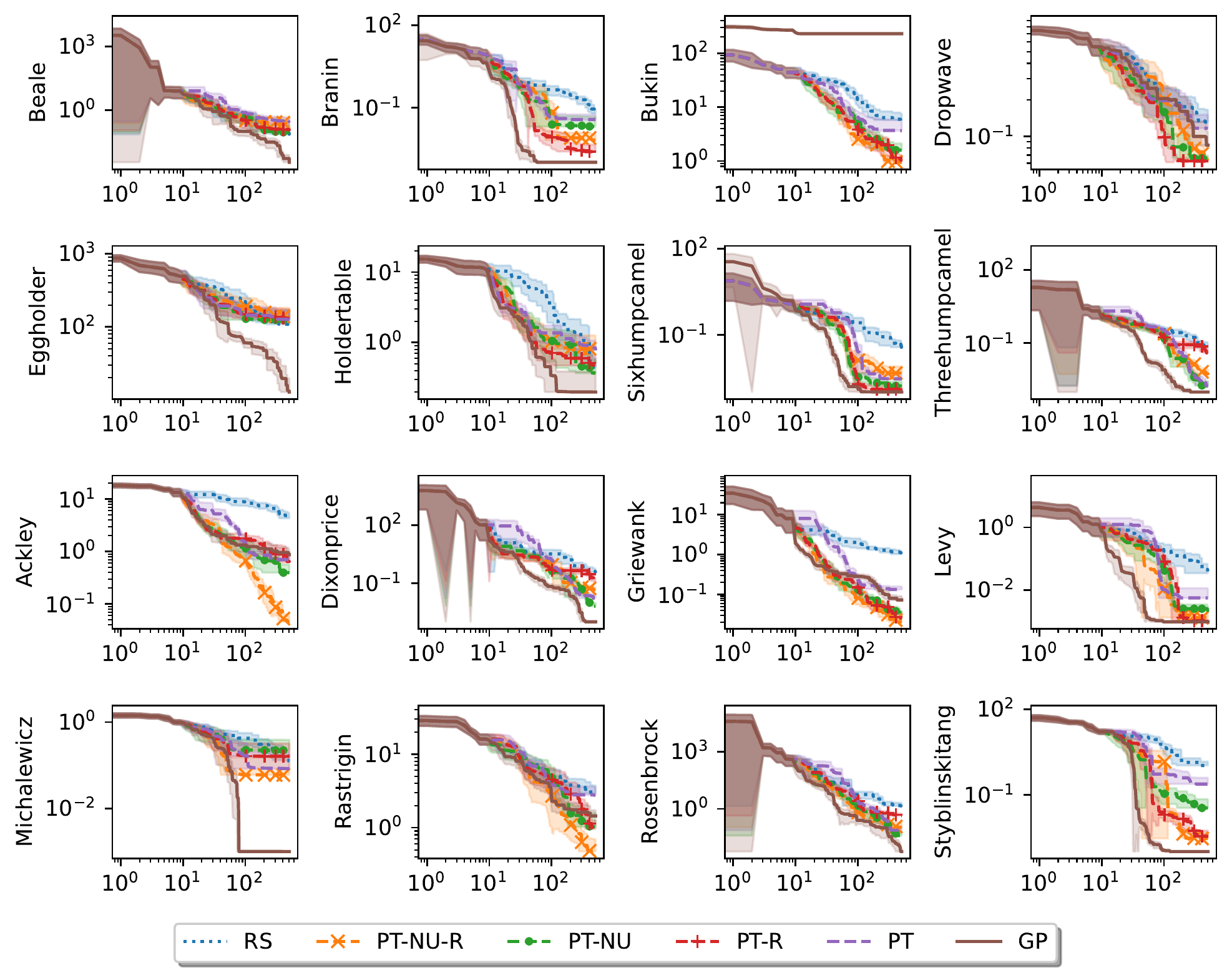}
    \caption{Simple regret results versus the number of samples time in 2 dimensions.}
    \label{fig:bo_all_runs_2d_sample}
\end{figure}

\begin{figure}[H]
    \centering
    \includegraphics[width=0.9\textwidth]{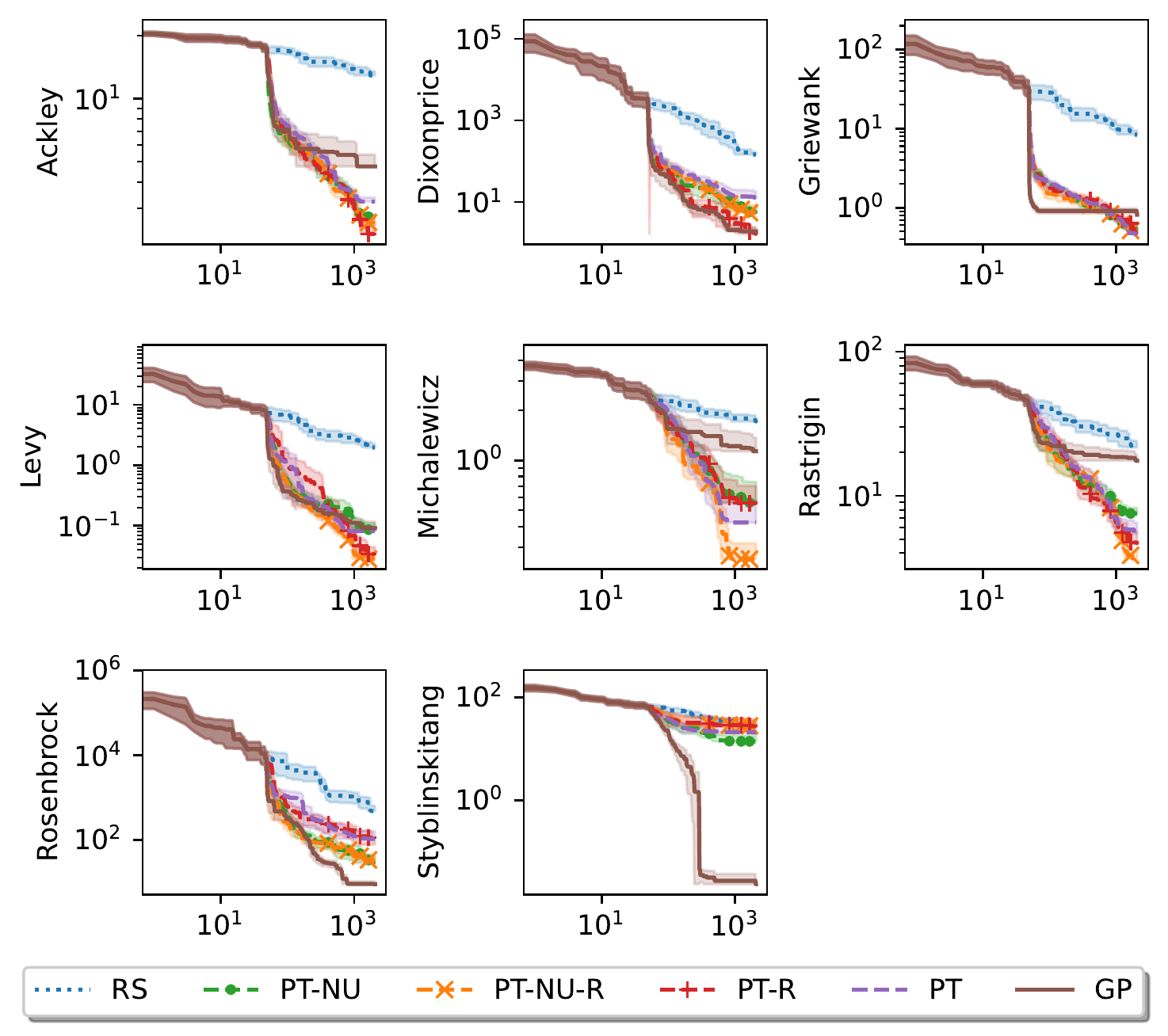}
    \caption{Simple regret results versus the number of samples in 5 dimensions.}
    \label{fig:bo_all_runs_5d_sample}
\end{figure}

\begin{figure}[H]
    \centering
    \includegraphics[width=0.9\textwidth]{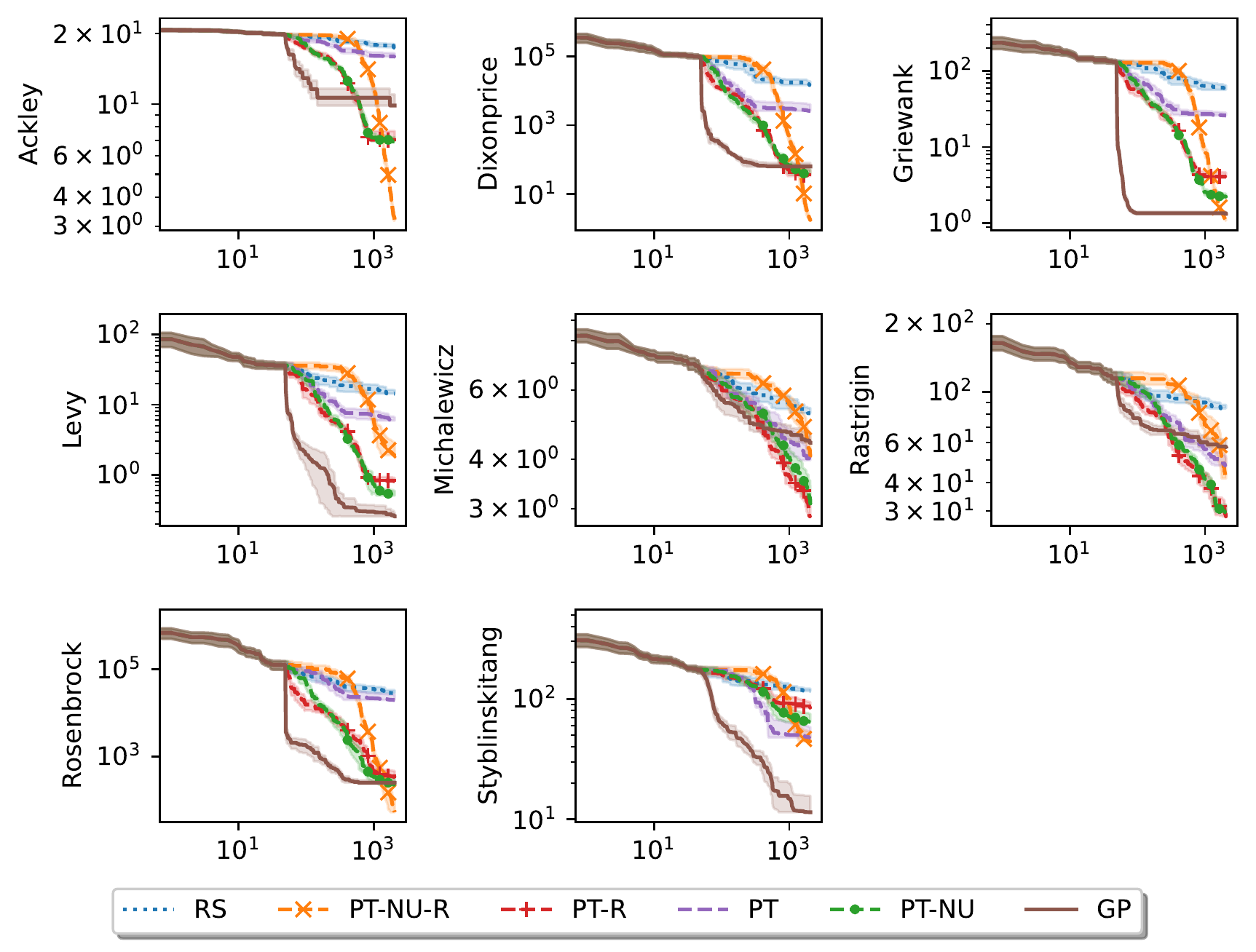}
    \caption{Simple regret results versus the number of samples in 10 dimensions.}
    \label{fig:bo_all_runs_10d_sample}
\end{figure}

\subsection{Ablation study}\label{app:ablation}
In this section, we discuss the influence of the hyperparameter $\varepsilon$ on the BO performance of the transformer trained with regularisation, i.e. PT-$\mathcal{R}_{\varepsilon}$.

In Figures \ref{fig:ablation_r_1d}, \ref{fig:ablation_r_2d}, \ref{fig:ablation_r_5d} and \ref{fig:ablation_r_10d} we display the simple regret results of BO runs with the method PT-$\mathcal{R}_{\varepsilon}$ for various values of this hyperparameter.
All plots are against the number of samples.
We show 5 functions that are available in 1, 2, 5 and 10 dimensions, i.e. Ackley, Dixonprice, Griewank, Levy, and Rastrigin.
We also display PT for reference.
\begin{figure}[H]
\centering
\begin{subfigure}[b]{0.99\textwidth}
   \includegraphics[width=1\linewidth]{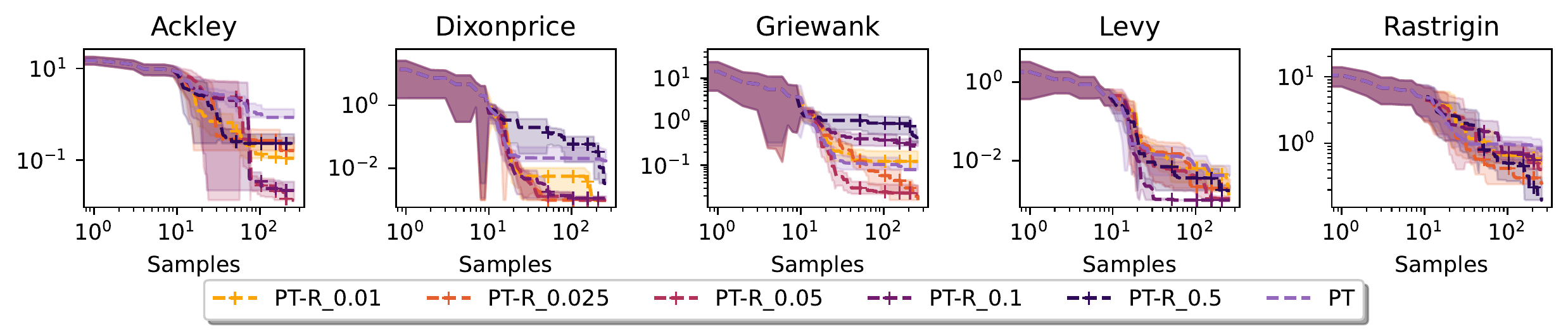}
  \caption{1 dimension with $\varepsilon \in \{0.01, 0.025, 0.05, 0.1, 0.5\}$}
   \label{fig:ablation_r_1d} 
\end{subfigure}

\begin{subfigure}[b]{0.99\textwidth}
   \includegraphics[width=1\linewidth]{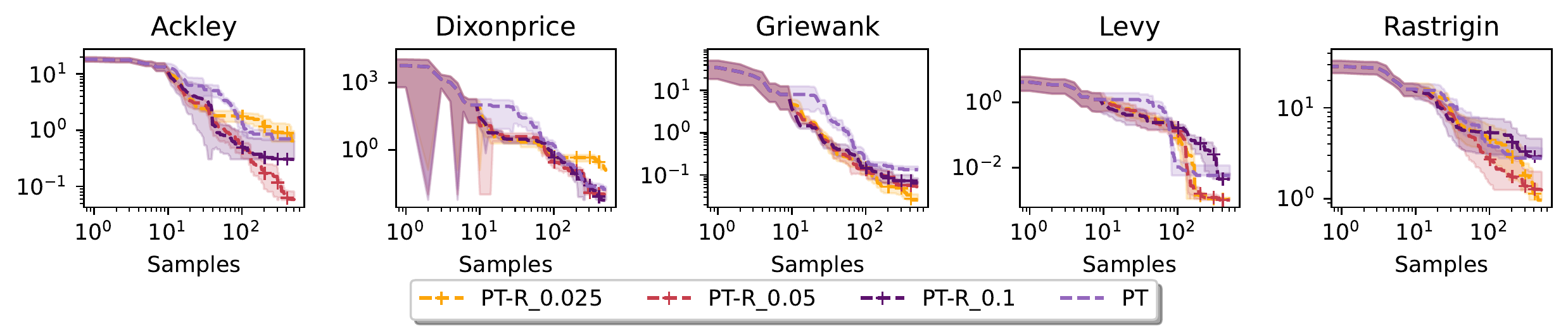}
  \caption{2 dimensions for $\varepsilon \in \{0.025, 0.05, 0.1\}$}
   \label{fig:ablation_r_2d}
\end{subfigure}

\begin{subfigure}[b]{0.99\textwidth}
   \includegraphics[width=1\linewidth]{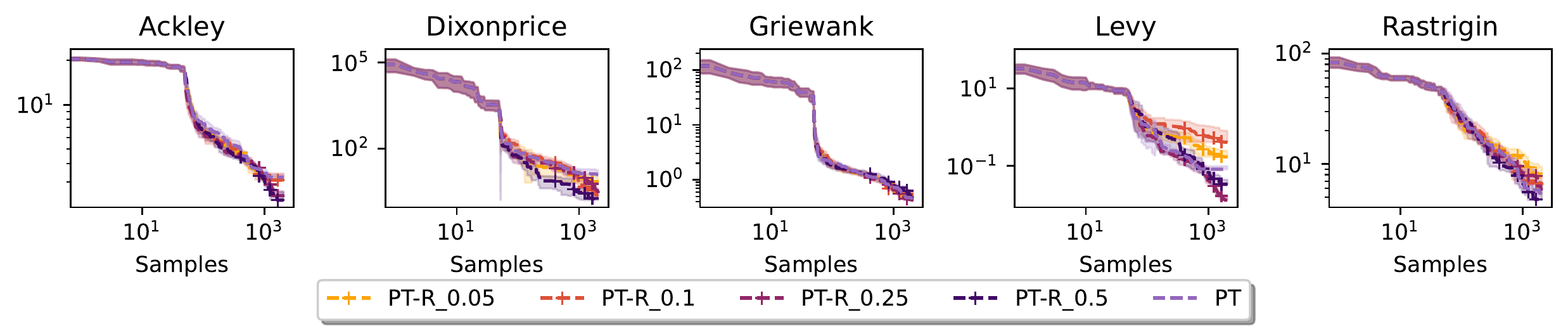}
  \caption{5 dimensions for $\varepsilon \in \{0.05, 0.1, 0.25, 0.5\}$}
   \label{fig:ablation_r_5d}
\end{subfigure}

\begin{subfigure}[b]{0.99\textwidth}
   \includegraphics[width=1\linewidth]{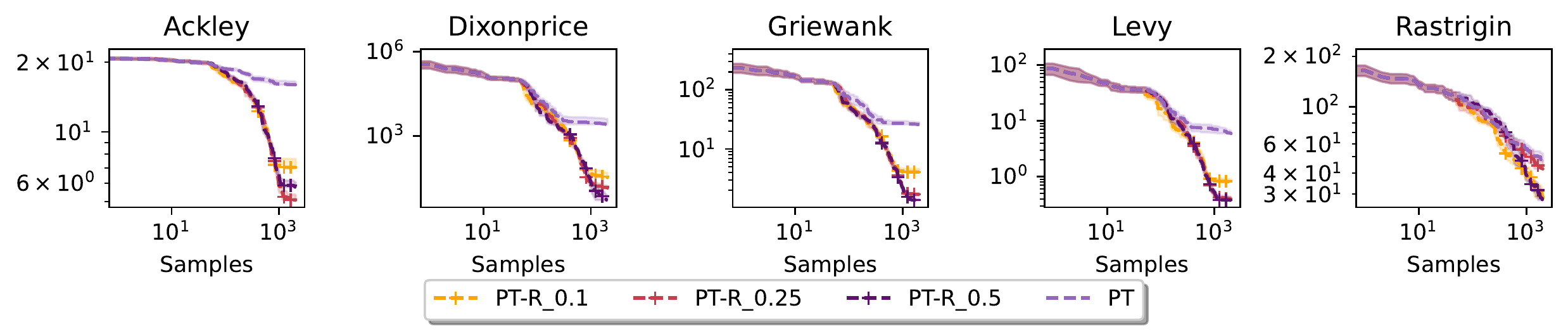}
  \caption{10 dimensions for $\varepsilon \in \{0.1, 0.25, 0.5\}$}
   \label{fig:ablation_r_10d}
\end{subfigure}

\caption[Ablation study on $\varepsilon$]{Simple regret results of PT and PT-$\mathcal{R}_{\varepsilon}$ plot against number of black-box samples with various values of $\varepsilon$ and different dimensions. For each dimension \ref{fig:ablation_r_1d}-\ref{fig:ablation_r_10d}, we trained PT-$\mathcal{R}_{\varepsilon}$ varying the value of $\varepsilon$ to showcase the effect of this component. While in a small number of cases (e.g. Dixonprice or Griewank in 5 dimensions \ref{fig:ablation_r_5d}) the value of that hyperparameter seems to be relatively unimportant, in other cases it makes a clear difference in the simple regret results. From these ablation plots we observe that progressively increasing the value of $\varepsilon$ as the dimension augments seems to be a good choice overall.
We refer the reader to Table~\ref{tab:hyperparam} in Section~\ref{app:hyperparam} for default values we used as they provide best average performance, though fine-tuning might be beneficial depending on the task.
}
\end{figure}

From those plots we learn that, while the regulariser is effective in improving the performance of the PT in BO, an \emph{over}-regularised training could potentially lead to no improvements.
For example in 2 dimensions, PT-$\mathcal{R}_{0.1}$ is less sample-efficient than PT-$\mathcal{R}_{0.05}$.
Respectively, we note that \emph{under}-regularising the training can also impact sample-efficiency as observed in Figure \ref{fig:ablation_r_5d} or Figure \ref{fig:ablation_r_10d} where PT-$\mathcal{R}_{0.05}$ and PT-$\mathcal{R}_{0.1}$ underperform compared to PT-$\mathcal{R}_{0.25}$ and PT-$\mathcal{R}_{0.5}$.

Choosing the right value of $\varepsilon$ might depend on the downstream tasks and on the assumptions made on the objective function.
But in general, our observation is to avoid over/under-regularisation, one should increase the value of $\varepsilon$ as the dimension of the problem increases, which is what we report in Table~\ref{tab:hyperparam}. 
Of course, one could tune this hyperparameter but the approach would lose in generality.

\subsection{The effect of non-uniform sampling and regularisation}
\label{app:visualExps}
In this section we illustrate the effect of the introduced components.
We want to compare the advantage of using a model trained with our components over using the original probabilistic transformer in a real BO experiment.
In order to compare PT and PT-$\nu\mathcal{R}_{\varepsilon}$ on similar data, we propose the following experimental setup.
\begin{wrapfigure}{r}{0.4\textwidth}
  \begin{center}
    \includegraphics[width=0.4\textwidth]{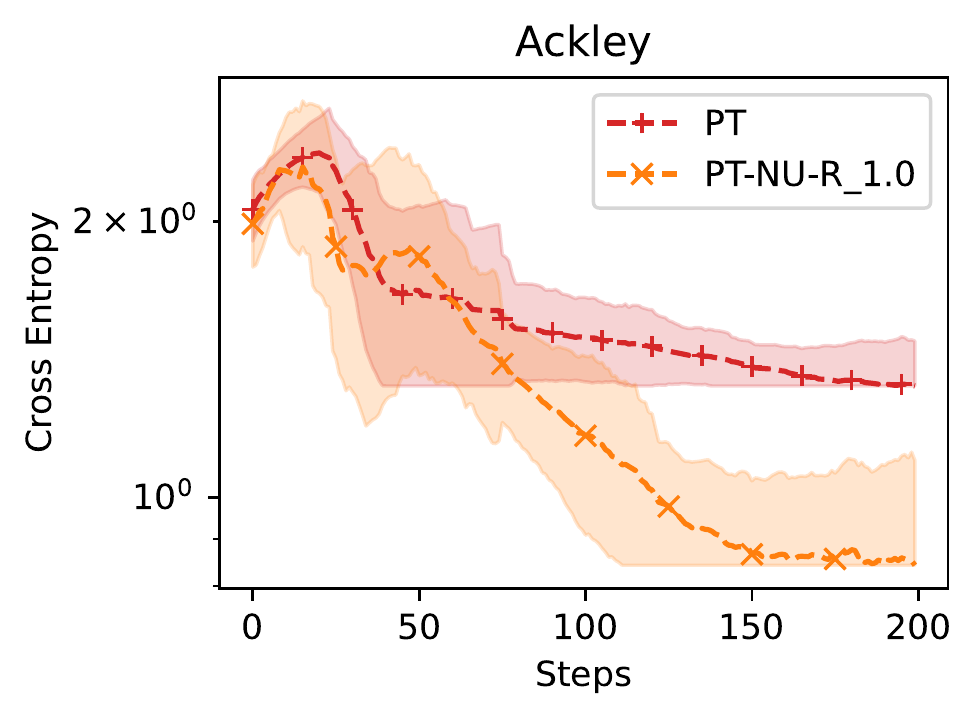}
  \end{center}
  \caption{Average cross-entropy for PT and PT-$\nu\mathcal{R}_{\varepsilon}$ on the acquired points of a reference GP-BO run. The experiment is repeated 10 times to create error bars.}
  \label{fig:ce_pred_visualisation}
\end{wrapfigure}

Considering Ackley function in 10 dimensions,
we first perform a GP-based BO run and record the initial points $D_{\text{init}} = \{\bm{x}_i^{\text{init}}, \bm{y}_i^{\text{init}}\}_{i=1}^{n_{\text{init}}}$ as well as all the acquired observations during this optimisation loop $D_{\text{acq}} = \{\bm{x}_i^{\text{acq}}, \bm{y}_i^{\text{acq}}\}_{i=1}^{n_{\text{acq}}}$.
We then use those points as train and test points for PT and PT-$\nu\mathcal{R}_{\varepsilon}$ in the following way.
At the first iteration, we give to both models only the initial points $D_{\text{init}}$. 
We make predictions at all the acquired points $D_{\text{acq}}$ and measure the average cross-entropy for both models.
At the next step, we make predictions and compute the average cross-entropy on $\{\bm{x}_i, \bm{y}_i\}_{i=2}^{n_{\text{acq}}}$ with both PT and PT-$\nu\mathcal{R}_{\varepsilon}$ having observed $D_{\text{init}} \cup \{\bm{x}_1^{\text{acq}}, \bm{y}_1^{\text{acq}}\}$.
We continue this process, progressively observing more data from $D_{\text{acq}}$ and evaluating the model's accuracy on the remaining points.

Figure \ref{fig:ce_pred_visualisation} shows the result of that experiment.
We observe that as we progressively add more GP-acquired points while predicting on the remaining, the PT-$\nu\mathcal{R}_{\varepsilon}$ run is able to make more accurate predictions early on.
This illustrates the advantage of our introduced components as the model trained with them is able to identify promising regions early compared to the original model.

\subsection{Hyperparameters}
\label{app:hyperparam}
In this section we report all settings and hyperparameters used for BO runs (see Table~\ref{tab:hyperparam_bo}) and PT training (see Table~\ref{tab:hyperparam}).

\begin{table}[H]
    \centering
    \caption{Hyperparameters of BO Runs.}
    \label{tab:hyperparam_bo}
    \begin{tabular}{ccccc}
    \toprule
    Hyperparameter & 1d & 2d & 5d & 10d \\
    \midrule
    number of seeds & \multicolumn{4}{c}{10} \\
    number of steps & 250 & 500 & 2000 & 2000 \\
    $n_{\text{init}}$ & 10 & 10 & 50 & 50 \\
    Acquisition & \multicolumn{4}{c}{Expected Improvement}\\
    \midrule
    Local test points & \multicolumn{4}{c}{1000} \\
    Initial perturbation & 0.25 & 0.25 & 0.25 & 0.5 \\
    Perturbation decay & 0.995 & 0.995 & 0.999 & 0.998\\
    \midrule
    Restarts & \multicolumn{4}{c}{10} \\
    Raw samples & \multicolumn{4}{c}{500} \\
    GP retrain from scratch & \multicolumn{4}{c}{every 10\% acquired points added}\\ 
    \bottomrule
    \end{tabular}
\end{table}

\begin{table}[H]
    \centering
    \caption{Hyperparameters of PT training.}
    \label{tab:hyperparam}
    \begin{tabular}{cp{1.6cm}p{1.6cm}p{1.6cm}p{1.6cm}}
    \toprule
    Hyperparameter & 1d & 2d & 5d & 10d \\
    \midrule
    Learning rate & \multicolumn{4}{c}{0.003} \\
    Learning rate decay & \multicolumn{4}{c}{Linear warmup during 50 epochs and cosine decay} \\
    Adam betas & \multicolumn{4}{c}{(0.9, 0.999)} \\
    L2 regularisation & \multicolumn{4}{c}{0.01} \\
    Activation functions & \multicolumn{4}{c}{Gaussian Error Linear Unit} \\
    Number of buckets & \multicolumn{4}{c}{100} \\
    Points estimate bucket size & \multicolumn{4}{c}{Batch size $\times$ 10 $\times$ Size of dataset } \\
    Number of layers & \multicolumn{4}{c}{6} \\
    Number of attention heads & \multicolumn{4}{c}{4} \\
    Embedding size & \multicolumn{4}{c}{512} \\
    Training epochs & \multicolumn{4}{c}{400} \\
    Epoch size & \multicolumn{4}{c}{10 training steps} \\
    Batch size & 12 & 12 & 4 & 4 \\
    Size of dataset & 2000 & 2000 & 4500 & 4500 \\
    $\varepsilon$ & 0.05 & 0.05 & 0.5 & 1.0 \\
    \bottomrule
    \end{tabular}
\end{table}

\subsection{Average running time and resources}\label{app:time}
In this section we report average running times for the training of the probabilistic transformer and the BO. 
Note that we report in Table \ref{tab:time_per_iteration} the BO-related times in seconds per iteration as not all experiments have the same number of runs.
Table \ref{tab:time_per_iteration} shows the order of magnitude acceleration from running BO with the PT instead of the GP (even if not retrained from scratch at every step).
\begin{table}[H]
    \centering
    \caption{Average running times (second per iteration).}
    \begin{tabular}{c|cc}
        \toprule
        Process &  PT-$\nu\mathcal{R}_{\varepsilon}$-BO & GP-BO \\
        Average time (sec/it) & 1.58$\cdot10^{-2}$ & 2.9$\cdot10^{-1}$ \\
        \bottomrule
    \end{tabular}
    \label{tab:time_per_iteration}
\end{table}

To clarify how the entries of Table~\ref{tab:time_per_iteration} are obtained, we now consider the runs in 2 dimensions. 
We trained PT-$\nu\mathcal{R}_{\varepsilon}$ for approximately 3 hours.
The BO runs then took on average of 3.8 seconds for the 500 BO steps.
So for 10 seeds on 16 tasks, the BO runs took approximately 10 minutes for a total of about 3.2 hours.
On the other hand, on average the GP-based BO run takes 105 seconds for 500 steps.
This totals to about 4.7 hours.

Take another example in 10 dimensions where we run 8 tasks for 2000 steps over 10 seeds.
We train PT-$\nu\mathcal{R}_{\varepsilon}$ for about 5.7 hours.
The BO runs take on average 26.7 seconds for 2000 steps compounding to 1.2 hours adding to the 5.7 hours of training, i.e. approximately 6.9 hours in total.
Respectively the average running time for one BO run with a GP model is 590 seconds.
This sums up to 26.2 hours in total.

We can see the significant time gain obtained by training a PT once and simply using forward passes during a BO run compared to using a task-specific model that is periodically retrained on newly acquired points. The more tasks we need to run, the higher the dimension and the larger the number of steps, the more attractive our model becomes.
The previous examples and the other dimensions are summarised in Table \ref{tab:time_total}.
\begin{table}[H]
    \centering
    \caption{Approximate total running times, including training for PT methods (hours).}
    \begin{tabular}{c|cccc}
        \toprule
        Process & 1D (6 tasks) & 2D (16 tasks) & 5D (8 tasks) & 10D (8 tasks) \\
        \midrule
        PT-$\nu\mathcal{R}_{\varepsilon}$-BO & 3.1 & 3.2 & 5.7 & 6.9 \\
        GP-BO & 3.2 & 4.7 & 21.4 & 26.2 \\
        \bottomrule
    \end{tabular}
    \label{tab:time_total}
\end{table}

We also report the hardware used for our experiments.
All BO runs and transformers trainings were done on one GPU, an Nvidia Tesla V100 SXM2 16GB.

Finally, please note that our provided code builds on the code released by \citet{2021_Muller} which can be found at \url{https://github.com/automl/TransformersCanDoBayesianInference}.

\end{document}